\documentclass{article}

\usepackage{microtype}
\usepackage{graphicx}
\usepackage{subcaption}
\usepackage{booktabs} 

\usepackage{hyperref}

\usepackage[preprint]{icml2026}

\usepackage{amsmath}
\usepackage{amssymb}
\usepackage{mathtools}
\usepackage{amsthm}

\usepackage[capitalize,noabbrev]{cleveref}

\theoremstyle{plain}
\newtheorem{theorem}{Theorem}[section]

\newtheorem{lemma}[theorem]{Lemma}
\newtheorem{corollary}[theorem]{Corollary}
\theoremstyle{definition}

\theoremstyle{remark}
\newtheorem{remark}[theorem]{Remark}

\usepackage[textsize=tiny]{todonotes}

\usepackage{amsmath,amsfonts,bm}

\def\eqref#1{equation~\ref{#1}}

\def\1{\bm{1}}

\def\rva{{\mathbf{a}}}

\def\rvs{{\mathbf{s}}}

\DeclareMathAlphabet{\mathsfit}{\encodingdefault}{\sfdefault}{m}{sl}
\SetMathAlphabet{\mathsfit}{bold}{\encodingdefault}{\sfdefault}{bx}{n}

\def\gA{{\mathcal{A}}}
\def\gB{{\mathcal{B}}}
\def\gC{{\mathcal{C}}}
\def\gD{{\mathcal{D}}}

\def\gH{{\mathcal{H}}}

\def\gL{{\mathcal{L}}}
\def\gM{{\mathcal{M}}}

\def\gR{{\mathcal{R}}}
\def\gS{{\mathcal{S}}}
\def\gT{{\mathcal{T}}}

\newcommand{\E}{\mathbb{E}}
\newcommand{\Ls}{\mathcal{L}}

\DeclareMathOperator*{\argmax}{arg\,max}
\DeclareMathOperator*{\argmin}{arg\,min}

\newcommand{\te}[1]{\texttt{#1}}

\usepackage[utf8]{inputenc} 
\usepackage[T1]{fontenc}    
\usepackage{hyperref}       
\usepackage{url}            
\usepackage{booktabs}       
\usepackage{amsfonts}       
\usepackage{nicefrac}       
\usepackage{microtype}      
\usepackage{xcolor}         

\usepackage{amsmath}
\usepackage{amssymb}
\usepackage{mathtools}
\usepackage{amsthm}
\usepackage{algorithm}
\usepackage{algorithmic}
\usepackage{graphicx}
\usepackage{booktabs} 
\usepackage{makecell}
\usepackage{multirow}
\usepackage{listings}
\usepackage{pythonhighlight}
\usepackage{wrapfig}

\usepackage{pifont}
\usepackage{duckuments}
\usepackage[makeroom]{cancel}

\newcommand{\PreserveBackslash}[1]{\let\temp=\\#1\let\\=\temp}
\newcolumntype{C}[1]{>{\PreserveBackslash\centering}p{#1}}
\newcolumntype{R}[1]{>{\PreserveBackslash\raggedleft}p{#1}}
\newcolumntype{L}[1]{>{\PreserveBackslash\raggedright}p{#1}}

\icmltitlerunning{Q-learning Penalized Transformer for Safe Offline Reinforcement Learning}

\begin{document}

\twocolumn[
  \icmltitle{Q-learning Penalized Transformer for Safe Offline Reinforcement Learning}



  \begin{icmlauthorlist}
    \icmlauthor{Shengchao Hu}{sjtu}
    \icmlauthor{Peng Wang}{zs}
    \icmlauthor{Jifeng Hu}{jl}
    \icmlauthor{Qiyang Zhou}{zs}
    \icmlauthor{Anning Hu}{sjtu}
    \icmlauthor{Li Shen}{zs}
    \icmlauthor{Ya Zhang}{sjtu}
    \icmlauthor{Dacheng Tao}{ntu}
  \end{icmlauthorlist}

  \icmlaffiliation{sjtu}{Shanghai Jiao Tong University, China}
  \icmlaffiliation{zs}{Shenzhen Campus of Sun Yat-sen University, China}
  \icmlaffiliation{jl}{Jilin University, China}
  \icmlaffiliation{ntu}{Nanyang Technological University, Singapore}

  \icmlcorrespondingauthor{Li Shen}{mathshenli@gmail.com}

  \icmlkeywords{Machine Learning, ICML}

  \vskip 0.3in
]



\printAffiliationsAndNotice{}  

\begin{abstract}
This paper addresses the problem of safe offline reinforcement learning, which involves training a policy to satisfy safety constraints using an offline dataset. 
This problem is inherently challenging as it requires balancing three highly interconnected and competing objectives: satisfying safety constraints, maximizing rewards, and adhering to the behavior regularization imposed by the offline dataset.
To tackle this trilogy challenge, we propose Q-learning Penalized Transformer policy (QPT), a \emph{training--inference consistent} framework that bridges conditional sequence modeling with constraint-aware value estimation.
QPT trains a Transformer policy that generates actions conditioned on trajectory context and target return/cost, retaining strong behavior regularization.
To inject explicit safety semantics during learning, we augment sequence-model training with a Q-shaped penalty using learned reward and cost Q-functions to favor high return under low constraint violation.
At inference, the same Q-functions enforce the cost threshold and choose the highest-reward feasible action, closing the loop between training and deployment.
We provide a principled analysis under stylized near-deterministic CMDPs, characterizing how Q-penalized conditional generation improve safety and performance.
Empirically, QPT consistently outperforms strong safe offline RL baselines across 38 tasks on the DSRL benchmark, and exhibits robust zero-shot adaptation to different constraint thresholds.

\end{abstract}
\section{Introduction}
\label{sec:Intro}

Offline reinforcement learning (RL) focuses on learning effective policies entirely from previously collected data, without requiring interaction with the environment. This paradigm has emerged as a powerful approach for addressing sequential decision-making tasks, such as autonomous driving \citep{hu2022st} and control systems \citep{zhan2022deepthermal}.
Various paradigms have been developed to maximize the utility of pre-collected trajectories while mitigating policy overfitting \citep{BEAR, fujimoto2019off, kostrikov2021offline, CQL}. However, standard offline RL often falls short in real-world applications, where diverse safety constraints limit feasible solutions, making the mere maximization of a scalar reward function insufficient.
The requirement for safety, or the satisfaction of constraints, is particularly critical when deploying RL algorithms in real-world scenarios \citep{garcia2015comprehensive}. Ensuring constraint satisfaction not only broadens the applicability of RL methods but also enhances their reliability in safety-critical domains.

Developing an optimal policy within a constrained manifold has been a central focus of recent research in safe offline RL, which seeks to integrate safety requirements into offline RL frameworks \citep{garcia2015comprehensive}.
Several approaches bridge concepts from offline RL and safe RL, employing techniques such as pessimistic estimations \citep{xu2022constraints} and stationary distribution correction \citep{lee2022coptidice}. Constrained optimization formulations, often incorporating Lagrange multipliers, are commonly used to identify policies that maximize rewards while adhering to safety constraints \citep{le2019batch}.
Sequential modeling methods, such as the Transformer \citep{liu2023constrained} and Diffuser \citep{lin2023safe, zheng2024safe}, have also been explored, demonstrating promising results in achieving both optimal policies and satisfying safety requirements.

However, the challenges of safe offline RL are amplified by the offline setting, which necessitates behavior regularization to mitigate distributional shift \citep{fujimoto2019off}. Balancing constraint satisfaction, reward maximization, and offline policy regularization is particularly difficult due to the intricate inter-dependencies among these objectives. Jointly optimizing them often results in unstable training and suboptimal safety performance \citep{lee2022coptidice, zheng2024safe}.
Furthermore, these objectives may inherently conflict \citep{xu2022constraints}. For instance, adhering to offline policy regularization can compromise constraint satisfaction when the dataset includes unsafe trajectories. Conversely, excluding unsafe trajectories may lead to suboptimal policies by omitting critical high-reward data, underscoring the inherent trade-offs in safe offline RL.
Besides, Lagrange-based methods frequently integrate the constraint threshold as a constant within the training process, seeking to optimize policy performance while adhering to specified constraints \citep{le2019batch}. We argue that the ability to adapt a trained policy to varying constraint thresholds is crucial for a wide range of real-world applications. In practice, enforcing stricter constraints typically results in diminished task performance and induces more conservative agent behaviors \citep{liu2022robustness}.  
Consequently, our objective is to investigate a training paradigm that enables an agent to dynamically adjust its constraint threshold at deployment. This approach would allow for flexible control over the agent's level of conservativeness, eliminating the need for additional fine-tuning or retraining.

To address these challenges, we propose Q-learning Penalized Transformer policy (QPT), a \emph{training--inference consistent} framework that couples conditional sequence modeling with constraint-aware value estimation.
QPT trains a Transformer policy that generates actions conditioned on trajectory context and target return/cost specifications, retaining strong behavior regularization and enabling flexible conditioning at test time.
To inject explicit safety semantics beyond imitation, we augment sequence-model training with a Q-shaped objective using learned reward and cost Q-functions, which biases generation toward high return with low constraint violation.
At inference, the same Q-functions are reused to certify feasibility under a given cost threshold and select the highest-reward feasible action, closing the loop between optimization and deployment.
On the theory side, we provide a principled analysis under stylized near-deterministic CMDPs, characterizing how Q-penalized conditional generation improve safety and performance in offline settings.
Experimental evaluations on 38 tasks from the DSRL benchmark \citep{liu2023datasets} further validate QPT's effectiveness, showing its ability to learn safe, robust, and high-reward policies. 
QPT consistently outperforms strong safe offline RL baselines in both reward and constraint satisfaction, and demonstrates robust zero-shot adaptation to varying constraint thresholds.


\section{Related Work}
\label{sec:ReW}

\textbf{Offline RL} trains policies from a static offline dataset $\gD$, without online interaction \citep{levine2020offline}, making it ideal for scenarios where interaction is costly or unsafe. A major challenge is \textit{distribution shift}, where the learned policy deviates from the behavior policy, causing performance degradation \citep{fujimoto2019off}.
To address this, prior works have employed constrained or regularized dynamic programming to limit policy deviations \citep{TD3BC, CQL, IQL}. 
Conditional sequence modeling predicts future actions from past experiences, constraining the policy within behavior boundaries and enabling zero-shot adaptability \citep{DT, GDT, hu2025tackling, hu2025analytic, PTDT, CommFormer}.

\textbf{Safe RL} involves learning policies that maximize long-term rewards while satisfying safety constraints \citep{wachi2020safe, gu2024review, liu2025safe}. 
%
A common approach to constrained optimization in safe RL is the primal-dual framework, which reformulates the problem into an unconstrained optimization using Lagrangian multipliers \citep{chen2021primal}. Correction-based methods provide another solution by projecting unsafe actions onto safe sets, incorporating domain knowledge to improve exploration safety \citep{zhao2021model, luo2021learning}.
Model-based RL has also been applied to improve data efficiency and performance \citep{huang2023safe}, though it often requires larger models to parameterize environment dynamics, increasing computational complexity.

\textbf{Safe offline RL} has also received growing attention, with the goal of ensuring zero constraint violations during inference. These methods combine offline RL techniques with safety constraints, such as using DICE-style approaches for constrained optimization \citep{polosky2022constrained, lee2022coptidice, kim2025semi} and Lagrangian-based methods for simplicity and compatibility with existing offline RL frameworks.
%
Recent studies have introduced novel networks for safe offline RL \citep{koirala2024latent, koirala2024fawac, gong2025offline}. For instance, CDT \citep{liu2023constrained, wang2024safe} employs sequential modeling to learn from trajectory datasets, while TREBI \citep{lin2023safe} and FISOR \citep{zheng2024safe} utilize diffusion models for safe policy development. TREBI generates safe trajectories directly, whereas FISOR uses a diffusion actor to constrain actions within feasible regions.
Beyond model-architecture advances, a complementary line of work improves safety from a data-centric perspective, e.g., by constructing or augmenting datasets to mitigate distributional shift and reduce out-of-distribution (OOD) behaviors. \citep{guo2025constraint, yao2024oasis, gong2025offline}.
In contrast to these methods, we propose a framework that explicitly casts reward maximization and cost control as \emph{controllable loss terms}. This formulation provides a direct mechanism to trade off performance and safety during optimization, yielding policies that simultaneously achieve high reward while maintaining low cost.

\section{Methodology}
\label{sec:method}

\subsection{Problem Setup}

RL problems with safety constraints are naturally formulated within the Constrained Markov Decision Process (CMDP) framework \citep{altman1998constrained}. A CMDP is defined by a tuple $\gM = (\gS, \gA, \gT, \gR, \gC, \mu_0)$, where $\gS$ is the state space, $\gA$ is the action space, $\gT: \gS \times \gA \times \gS \to [0, 1]$ is the state transition probability function, $\gR: \gS \times \gA \to \mathbb{R}$ defines the reward function, and $\gC: \gS \times \gA \to [0, C_{\text{max}}]$ quantifies the costs associated state-action pairs, where \(C_{\text{max}}\) is the maximum possible cost, and $\mu_0: \gS \rightarrow [0, 1]$ is the initial state distribution. 

In safe offline RL problems, we are given a fixed pre-collected dataset $\mathcal{D}$=$\{(\rvs, \rva, r, c, \rvs')_i\}_{i=1}^{\gH}$ from one or more (unknown) behavior policies, where each training example $i$ contains the action $\rva$ taken at state $\rvs$, reward received $r$, cost incurred $c$, and the next state $\rvs'$. The goal is to learn a policy $\pi: \gS \to \gA$ from the offline dataset $\mathcal{D}$ to maximize the expected reward while satisfying a specified cost/safety constraint. This problem is mathematically formulated as:
\begin{equation}
\label{eq:safe_obj}
    \max_{\pi} \mathbb{E}_{\tau \sim \pi}[R(\tau)] \quad \mathrm{subject\ to} \quad \mathbb{E}_{\tau \sim \pi}[C(\tau)] \leq \kappa.
\end{equation}
Here, $\kappa \in [0, +\infty)$ is the cost threshold for safety constraint, $\gH$ is the horizon length of episode, $\tau = \{\rvs_1, \rva_1, r_1, c_1, \ldots, \rvs_\gH, \rva_\gH, r_\gH, c_\gH\}$ denotes a trajectory sampled by executing the policy $\pi$, $R(\tau) = \sum_{t=1}^\gH r_t$ is the total accumulated reward, and $C(\tau) = \sum_{t=1}^\gH c_t$ is the total incurred cost.

Most existing offline safe RL methods approach policy training as a constrained optimization problem, wherein learnable dual variables are updated according to estimates of constraint violation costs and a target threshold \citep{xu2022constraints, lee2022coptidice, polosky2022constrained}. While this constrained optimization paradigm is effective in online safe RL settings \citep{stooke2020responsive}, it faces significant challenges in the offline context \citep{liu2023constrained}. 
First, offline RL policies often become either unsafe or overly conservative due to biased value estimates, stemming from incomplete dataset coverage. In the Lagrangian dual optimization of Equation \ref{eq:safe_obj}, such bias in cost estimation $C(\tau)$ can mislead dual variable updates relative to the fixed threshold $\kappa$, resulting in unsafe or overly cautious behaviors, a problem exacerbated in offline settings \citep{liu2022constrained}.
Second, policies cannot adapt to new constraint thresholds without retraining, as the threshold must remain fixed during training. Changing it post hoc destabilizes dual variables and can cause optimization to diverge, thus requiring full retraining for each new constraint.

To address these issues, we reformulate the learning objective described in Equation \ref{eq:safe_obj} and leverage sequential modeling techniques, which have shown promise in achieving zero-shot adaptation to varying constraint thresholds while maintaining near-optimal task performance \citep{hu2022transforming, liu2023constrained}.
However, a key limitation of sequence modeling in this context is its tendency to imitate the behavior distribution present in the training dataset \citep{RCSL}, which often includes unsafe trajectories. Directly learning from such datasets may therefore result in unsafe policies. One potential remedy is to filter out unsafe trajectories from the dataset $\mathcal{D}$ to construct a ``safe'' dataset. Unfortunately, this strategy often eliminates high-reward transitions, leading to suboptimal policies.
Ideally, a safe offline learner should exploit the dataset more selectively---preserving strong behavior regularization while biasing generation toward high-return, constraint-satisfying decisions, including ``stitching'' together safe segments from otherwise unsafe trajectories.
We therefore propose QPT, a training--inference consistent framework based on a conditional Transformer and reward/cost Q-functions.
QPT uses a Q-shaped objective to bias conditional imitation toward high return with low constraint violation, without discarding informative data.
At deployment, the same Q-functions certify feasibility under a given cost threshold and select the highest-reward feasible action, closing the loop between learning and decision making.
We also include a lightweight augmentation scheme for infeasible specifications and validate QPT theoretically and empirically (Sections~\ref{sec:CTP}--\ref{sec:theoretic}).

\subsection{Conditional Transformer Policy}
\label{sec:CTP}

The Transformer architecture \citep{vaswani2017attention}, extensively studied in NLP \citep{bert} and CV \citep{ViT}, has also been explored in RL through the conditional sequence modeling (CSM) paradigm \citep{hu2022transforming}.
In contrast to most traditional RL methods, which rely on value function estimation or policy gradient computation, DT \citep{DT} directly predicts desired future actions based on a sequence of historical data comprising state ($\rvs_t$), action ($\rva_t$), and return-to-go ($\hat{r}_t = \sum_{i=t}^T r_i$) tuples.
%
%
In the context of safe offline RL, this formulation is extended by including an additional cost-to-go token, $\hat{c}_t = \sum_{i=t}^T c_i$, which quantifies the cumulative cost from the current time step to the end of the episode \citep{liu2023constrained}.
During training on offline data, the Transformer processes trajectory sequences in an auto-regressive manner, utilizing a historical context of the most recent $K$ steps. A trajectory sequence $\tau_t$ is formulated as follows:
\begin{equation}
\label{eq:input}
    \tau_t = (\hat{r}_{t-K+1}, \hat{c}_{t-K+1}, \rvs_{t-K+1}, \rva_{t-K+1}, \dots,  \hat{r}_{t}, \hat{c}_{t}, \rvs_{t}, \rva_{t}).
\end{equation}

The prediction head corresponding to the state token $\rvs_t$ is trained to predict the associated action $\rva_t$.
For continuous action spaces, the training objective is to minimize the mean squared error (MSE) loss, defined as:
\begin{equation}
\label{eq:DTLoss}
   \Ls_{DT} = \mathbb{E}_{\tau_t \sim \gD} \left[ \frac{1}{K} \sum_{i=t-K+1}^t (\rva_i - \pi(\tau_t)_i )^2 \right],
\end{equation}
where $\pi(\tau_t)_i$ denotes the $i$-th action output of the policy $\pi$ learned by Equation \ref{eq:DTLoss}.

\subsection{Q-learning Penalization}
\label{sec:guidance}

To address the ``stitching'' challenge and design a target-conditioned policy that aligns the expected returns of sampled actions with the optimal returns while simultaneously minimizing the associated expected cost, we leverage the penalization from the Q-learning module \citep{kumar2022should, QT}.

In the safe offline RL setting, two types of Q-networks are utilized: the reward Q-network and the cost Q-network.
A straightforward approach to learning these networks involves applying the empirical Bellman evaluation operator, $\mathcal{T}^{\hat{\pi}}$, to samples $\left(\rvs, \rva, r, c, \rvs^{\prime} \right) \sim \mathcal{B}$:
\begin{align}
Q^r(\rvs, \rva)=r+\gamma\mathbb{E}_{\rva'\sim\hat{\pi}(\cdot|\rvs')}\left[Q^r(\rvs', \rva')\right], \\ 
Q^c(\rvs, \rva)=c+\gamma\mathbb{E}_{\rva'\sim\hat{\pi}(\cdot|\rvs')}\left[Q^c(\rvs', \rva')\right],
\end{align}
where $Q^r$ and $Q^c$ denote the reward and cost Q-networks, respectively, $\gamma$ represents the discount factor, and $\hat{\pi}$ represents the learned policy by Equation \ref{eq:final_update}. 

To mitigate overestimation bias, we employ the double Q-learning technique \citep{hasselt2010double}, constructing two Q-networks for each type: $Q^r_{\phi_1}, Q^r_{\phi_2}$ for the reward Q-network and $Q^c_{\psi_1}, Q^c_{\psi_2}$ for the cost Q-network.
These are accompanied by their corresponding target networks: $Q^r_{\phi_1'}, Q^r_{\phi_2'}, Q^c_{\psi_1'}, Q^c_{\psi_2'}$.
Additionally, we construct a target policy $\hat{\pi}_{\theta'}$ to guide the learning process.

Given that the input to the Transformer policy includes trajectory history, we adopt the \textit{n-step Bellman equation} to estimate the Q-networks. 
This choice is motivated by its demonstrated improvements over the 1-step approximation \citep{sutton2018reinforcement}.
The optimization of the reward Q-network parameters $\phi_i$ for $i \in \{1, 2\}$ is performed by minimizing the following objective:
\begin{align}
\label{eq:Q_update}
    &\E_{\tau_t \sim \gD, \hat{\rva}_t \sim \hat{\pi}_{\theta'}}  
      \sum_{m=t-K+1}^{t-1} 
    \Big|\Big| \hat{Q}^r_m - Q^r_{\phi_i}(\rvs_m, \rva_m) \Big|\Big|^2, \\ \nonumber
     &s.t.~~ \hat{Q}^r_m = \sum_{j=m}^{t-1} \gamma^{j-m} r_j
     + \gamma^{t-m} \min_{i=1,2} Q^r_{\phi_i'} (\rvs_t, \hat{\rva}_t),
\end{align}
where $\hat{\rva}_t$ denotes the predicted action output by the target policy $\hat{\pi}_{\theta'}$.
Similarly, the optimization of the cost Q-network parameters $\psi_i$ for $i \in \{1, 2\}$ is performed by minimizing the following objective:
\begin{align}
\label{eq:Qc_update}
    &\E_{\tau_t \sim \gD, \hat{\rva}_t \sim \hat{\pi}_{\theta'}}  
      \sum_{m=t-K+1}^{t-1} 
    \Big|\Big| \hat{Q}^c_m - Q^c_{\psi_i}(\rvs_m, \rva_m) \Big|\Big|^2, \\ \nonumber
     &s.t.~~ \hat{Q}^c_m = \sum_{j=m}^{t-1} \gamma^{j-m} c_j
     + \gamma^{t-m} \min_{i=1,2} Q^c_{\psi_i'} (\rvs_t, \hat{\rva}_t).
\end{align}

Leveraging the learned Q-networks, we incorporate them as penalization mechanisms during the training phase. This approach aims to enhance the policy's ``stitching'' capability by prioritizing the sampling of high-reward actions while ensuring low-cost trajectories are favored.
The final learning objective is formulated as a linear combination of MSE loss and Q-learning penalization terms:
\begin{align}
\label{eq:final_update}
    \hat{\pi} &= \argmin_{\hat{\pi}_{\theta}} \left\{\gL(\theta)
    := \gL_{DT} (\theta) - \gL_{Q^r} (\theta) + \gL_{Q^c} (\theta)\right\}\nonumber \\
    &= \! \argmin_{\hat{\pi}_{\theta}}  \underbrace{\gL_{DT}(\theta)}_{\text{regularization}}\!\!-\! \alpha_1 \cdot \underbrace{\E_{\tau_t \sim \gD} \E_{(\rvs_i, \rva_i) \sim \tau_t} Q^r_{\phi}(\rvs_i, \hat{\pi}(\tau_t)_i)}_{\text{reward maximization}}\nonumber \\
    &+ \alpha_2 \cdot \underbrace{\E_{\tau_t \sim \gD} \E_{(\rvs_i, \rva_i) \sim \tau_t} Q^c_{\psi}(\rvs_i, \hat{\pi}(\tau_t)_i)}_{\text{cost minimization}}.
\end{align}
To account for variations in the scale of the Q-networks across different offline datasets, we employ a normalization technique inspired by \citet{TD3BC}. Specifically, the weighting factors $\alpha_1$ and $\alpha_2$ are defined as follows:
\begin{align}
    \alpha_1 &= \frac{\eta_1}{\E_{\tau_t \sim \gD} \E_{(\rvs, \rva) \sim \tau_t} \left[ |Q^r_{\phi} (\rvs, \rva)| \right]}, \\
    \alpha_2 &= \frac{\eta_2}{\E_{\tau_t \sim \gD} \E_{(\rvs, \rva) \sim \tau_t} \left[ |Q^c_{\psi} (\rvs, \rva)| \right]},
\end{align}
where $\eta_1, \eta_2$ are hyperparameters that control the balance between these loss terms. 
Notably, these Q-networks in the denominator serve exclusively for normalization and are not subject to differentiation.

\subsection{Data Augmentation and Ensemble}
\label{sec:ensemble}
QPT leverages a conditional transformer structure, making the agent's behavior highly sensitive to the selection of target reward and cost values.
In the context of safe offline RL, the range of feasible and valid target cost and reward pairs is inherently limited.
This limitation poses a significant challenge: how can conflicts between the two target returns be effectively resolved while ensuring that meeting the target cost is prioritized over maximizing the target reward?
To address the aforementioned issues, we employ two techniques: data augmentation and ensemble.

Inspired by CDT \citep{liu2023constrained}, when an infeasible pair of target reward and cost $(\rho, \kappa)$ arises, we associate the conflicting target with the safest trajectory that achieves the maximum reward:
\begin{equation}
    \tau^* = \argmax_{\tau \sim \gD} R(\tau), s.t. ~ C(\tau) \leq \kappa.
\end{equation}
Based on the identified trajectory $\tau^* = \{\hat{r}_t^*, \hat{c}_t^*, \rvs_t^*, \rva_t^*\}_t$, we construct a new augmented trajectory:
\begin{equation}
\label{eq:dataaug}
    \hat{\tau} = \{\hat{r}_t^* + \rho - R(\tau^*), \hat{c}_t^* + \kappa - C(\tau^*), \rvs_t^*, \rva_t^* \}_t,
\end{equation}
where the operation over $\hat{r}^*$ and $\hat{c}^*$ are applied element-wise.
This augmentation technique enables the agent to learn by imitating the behavior of the most rewarding and safe trajectory $\tau^*$ when the desired target pair $(\rho, \kappa)$ is infeasible.
Further details on this process are provided in Appendix \ref{sec:app_DA}.

Moreover, sequence modeling methods are sensitive to the choice of target conditioning, which serves as input to the policy during inference.
Rather than manually tuning the values of the return-to-go and cost-to-go tokens, as required in previous conditional transformer policies -- a process that demands extensive trial and error -- we leverage learned reward and cost Q-networks to guide action selection. Specifically, actions are preferentially sampled to maximize expected returns while minimizing costs, following the approach in \citet{QT}. This process can be formulated as:
\begin{align}
    &\quad\quad \argmax_{\hat{\rva}_t^j} ~~ Q^r_{\phi'}( \rvs_t, \hat{\rva}_t^j), \label{eq:inf} \\
    &\quad\quad s.t. \quad Q^c_{\psi'} (\rvs_t, \hat{\rva}_t^j) \leq \hat{c}^j_{t}, \label{eq:inf_c}\\
    \hat{\rva}_t^j = \hat{\pi}(&\hat{r}^j_{t-K+1:t}, \hat{c}^j_{t-K+1:t}, \rvs_{t-K+1:t},\rva_{t-K+1:t-1} ))\nonumber.
\end{align}
Here, $(\hat{r}^j, \hat{c}^j)$ represent candidate target reward and cost pairs.
This approach is highly parallelizable. By assigning distinct return-to-go and cost-to-go pairs to each batch, we can effectively utilize GPU capabilities to concurrently generate multiple action sequences, thereby minimizing computational overhead. 
Further details on this process are provided in Appendix \ref{sec:app_ens}
A larger number of candidate target pairs provides a broader search space, potentially improving performance. However, this also incurs increased computational costs and greater susceptibility to noisy or suboptimal pairs, stemming from the biased estimation of the learned Q-networks.
Corresponding ablation studies are conducted to demonstrate the efficacy of this procedure, as detailed in Section \ref{sec:ab} and Appendix \ref{sec:abnumber}.
The training and inference procedures are thoroughly outlined in Algorithm \ref{alg:QPT}, providing a comprehensive summary of the processes involved.

\subsection{Theoretical Analysis}
\label{sec:theoretic}

We provide a theoretical justification for QPT in a stylized setting.
Under mild coverage and regularity assumptions, we show that the Q-penalized update in Equation \ref{eq:final_update} improves action selection compared to the standard sequence-model objective in Equation \ref{eq:DTLoss}, which yields corresponding improvements in expected return and expected cost.

\begin{theorem}
\label{thm:improve}
Consider an MDP with binary rewards and costs, behavior policy $\beta$, and conditioning function $f^r$ and $f^c$. Let $g^r(\tau) = \sum_{t=1}^\gH r_t, g^c(\tau) = \sum_{t=1}^\gH c_t$. Assume the following:
\begin{enumerate}
    \item Return coverage: $P_{\beta} (g^r(\tau) = f^r(\rvs_1) |\rvs_1 ) \geq \alpha_{f^r}, P_{\beta} (g^c(\tau) = f^c(\rvs_1) |\rvs_1 ) \geq \alpha_{f^c}$ for all initial states $\rvs_1$.
    \item Near determinism: $P( r \neq \gR(\rvs,\rva) ~or~ c \neq \gC(\rvs, \rva) ~or~ \rvs' \neq \gT(\rvs,\rva) | \rvs,\rva) \leq \epsilon$ at all $\rvs, \rva$ for some functions $\gT$, $\gR$ and $\gC$.
    \item Consistency of $f^r$ and $f^c$: $f^r(\rvs) = f^r(\rvs') + r, f^c(\rvs) = f^c(\rvs') + c$~ for all $\rvs$.
\end{enumerate}
For timestep $i$, the probabilities of selecting actions with maximum reward or minimum cost satisfy:

    1. \textbf{Reward Selection}: $P\{\hat{P}^r_i - P^r_i  \geq \sigma_r, \forall~i \} \geq 1 - \delta_r$, where $P_i^r$ and $\hat{P}^r_i$ are probabilities under the policies updated by Equation \ref{eq:DTLoss} and Equation \ref{eq:final_update}, respectively. With probability at least $(1 - \delta_r)$:
    \begin{align}
        \E_{\tau \sim \pi^*}[g^r(\tau)] - \E_{\tau \sim \hat{\pi}} [g^r(\tau)] \leq \epsilon (\frac{1}{\alpha_{f^r}} + 3) \gH^2 - \gH \sigma_r. \nonumber 
    \end{align}
    2. \textbf{Cost Selection}: $P\{\hat{P}^c_i - P^c_i  \geq \sigma_c, \forall~i \} \geq 1 - \delta_c$, where $P_i^c$ and $\hat{P}^c_i$ are probabilities under the policies updated by Equation \ref{eq:DTLoss} and Equation \ref{eq:final_update}, respectively. With probability at least $(1 - \delta_c)$:
    \begin{align}
        \E_{\tau \sim \hat{\pi}} [g^c(\tau)] - \E_{\tau \sim \pi^*}[g^c(\tau)] \leq \epsilon (\frac{1}{\alpha_{f^c}} + 3) \gH^2 - \gH \sigma_c. \nonumber
    \end{align}
\end{theorem}

\begin{table*}[!t]
\centering
\caption{ Complete evaluation results of the normalized reward and cost. The cost threshold is 1.
The $\uparrow$ symbol denotes that the higher reward, the better. The $\downarrow$ symbol denotes that the lower normalized cost (up to threshold 1, corresponding to original cost limit 10), the better. 
Each value is averaged over 20 evaluation episodes and 3 random seeds.
\textbf{Bold}: Safe agents whose normalized cost is smaller than 1. 
{\color[HTML]{656565} Gray}: Unsafe agents with normalized costs exceeding 1.
{\color[HTML]{0000FF} \textbf{Blue}}: Safe agent with the highest reward.} 
\label{tab:res}
\renewcommand{\arraystretch}{1.1}
\resizebox{1.\linewidth}{!}{
\begin{tabular}{|c|cc|cc|cc|cc|cc|cc|cc|cc|cc|}
\hline
& \multicolumn{2}{c|}{\textbf{QPT (Ours)}} & \multicolumn{2}{c|}{BC-Safe} & \multicolumn{2}{c|}{CDT} & \multicolumn{2}{c|}{BCQ-Lag} & \multicolumn{2}{c|}{CPQ} & \multicolumn{2}{c|}{COptiDICE} & \multicolumn{2}{c|}{FISOR} & \multicolumn{2}{c|}{OASIS} & \multicolumn{2}{c|}{CAPS} \\ \cline{2-19} 
\multirow{-2}{*}{Task} 
& reward $\uparrow$ & cost $\downarrow$ 
& reward $\uparrow$ & cost $\downarrow$ 
& reward $\uparrow$ & cost $\downarrow$ 
& reward $\uparrow$ & cost $\downarrow$ 
& reward $\uparrow$ & cost $\downarrow$ 
& reward $\uparrow$ & cost $\downarrow$ 
& reward $\uparrow$  & cost $\downarrow$  
& reward $\uparrow$  & cost $\downarrow$  
& reward $\uparrow$  & cost $\downarrow$  \\ \hline
PointButton1                                                         & {\color[HTML]{0000FF}\textbf{0.13}}   & {\color[HTML]{0000FF}\textbf{0.81}} & \textbf{0.10}  & \textbf{0.63} & {\color[HTML]{656565} 0.62}          & {\color[HTML]{656565} 7.17}          & {\color[HTML]{656565} 0.24}          & {\color[HTML]{656565} 1.73}    & {\color[HTML]{656565} 0.69}          & {\color[HTML]{656565} 3.2}           & {\color[HTML]{656565} 0.13}          & {\color[HTML]{656565} 1.35} & \textbf{0.03} & \textbf{0.81} & {\color[HTML]{656565} 0.26 } & {\color[HTML]{656565} 1.33 } & \textbf{0.02} & \textbf{0.30 }        \\
PointButton2                                                         &  \textbf{-0.01}  &  \textbf{0.88} & {\color[HTML]{0000FF} \textbf{0.04}}           & {\color[HTML]{0000FF} \textbf{0.58}}           & {\color[HTML]{656565} 0.31}          & {\color[HTML]{656565} 5.15}          & {\color[HTML]{656565} 0.4}           & {\color[HTML]{656565} 2.66}          & {\color[HTML]{656565} 0.58}          & {\color[HTML]{656565} 4.3}           & {\color[HTML]{656565} 0.15}          & {\color[HTML]{656565} 1.51}  & \textbf{0.02} & \textbf{0.69}  &  {\color[HTML]{656565} 0.31} & {\color[HTML]{656565} 1.89}  & \textbf{0.01} & \textbf{0.92 }      \\
PointCircle1  & {\color[HTML]{0000FF} \textbf{0.58}}  & {\color[HTML]{0000FF} \textbf{0.93}} & \textbf{0.45}  & \textbf{0.67} &  \textbf{0.57} & \textbf{0.75} & {\color[HTML]{656565} 0.17} & {\color[HTML]{656565} 1.04} & \textbf{0.43} & \textbf{0.29} & {\color[HTML]{656565} 0.78} & {\color[HTML]{656565} 15.64} & {\color[HTML]{656565} 0.60} & {\color[HTML]{656565} 12.8} 
& \textbf{0.32} & \textbf{0.00} & {\color[HTML]{656565} 0.25} & {\color[HTML]{656565} 2.10} \\
PointCircle2  & {\color[HTML]{0000FF} \textbf{0.62}}  & {\color[HTML]{0000FF} \textbf{0.92}} & \textbf{0.49}  & \textbf{0.44} & {\color[HTML]{656565} 0.61} & {\color[HTML]{656565} 1.39} & {\color[HTML]{656565} 0.53} & {\color[HTML]{656565} 8.35} & \textbf{0.28} & \textbf{0.77} & {\color[HTML]{656565} 0.78} & {\color[HTML]{656565} 25.94} & {\color[HTML]{656565} 0.70} & {\color[HTML]{656565} 11.79}
& \textbf{0.36} & \textbf{0.20} & \textbf{0.38} & \textbf{0.40} \\
PointGoal1    & {\color[HTML]{0000FF} \textbf{0.68}}  & {\color[HTML]{0000FF} \textbf{0.65}} & \textbf{0.42}  & \textbf{0.70} & {\color[HTML]{656565} 0.70} & {\color[HTML]{656565} 1.54} & {\color[HTML]{656565} 0.59} & {\color[HTML]{656565} 1.30} & \textbf{0.68} & \textbf{0.76} & {\color[HTML]{656565} 0.35} & {\color[HTML]{656565} 1.75} & {\color[HTML]{656565} 0.54} & {\color[HTML]{656565} 2.73}
& {\color[HTML]{656565} 0.66} & {\color[HTML]{656565} 4.35} & {\color[HTML]{656565} 0.22} & {\color[HTML]{656565} 1.05} \\
PointGoal2    & {\color[HTML]{0000FF} \textbf{0.18}}  & {\color[HTML]{0000FF} \textbf{0.76}} & \textbf{0.16}  & \textbf{0.29} & {\color[HTML]{656565} 0.57} & {\color[HTML]{656565} 3.45} & {\color[HTML]{656565} 0.71} & {\color[HTML]{656565} 7.53} & {\color[HTML]{656565} 0.08} & {\color[HTML]{656565} 2.14} & {\color[HTML]{656565} 0.42} & {\color[HTML]{656565} 2.71} & \textbf{0.04} & \textbf{0.14}
& {\color[HTML]{656565} 0.31} & {\color[HTML]{656565} 3.00} & {\color[HTML]{656565} 0.20} & {\color[HTML]{656565} 1.21} \\
PointPush1    & {\color[HTML]{0000FF}\textbf{0.34}} & {\color[HTML]{0000FF}\textbf{0.81}} & \textbf{0.17} & \textbf{0.72} & {\color[HTML]{656565} 0.23} & {\color[HTML]{656565} 1.65} & {\color[HTML]{656565} 0.19} & {\color[HTML]{656565} 1.05} & \textbf{0.21} & \textbf{0.29} & \textbf{0.12} & \textbf{0.82} & \textbf{0.28} & \textbf{0.54}
& {\color[HTML]{656565} 0.00} & {\color[HTML]{656565} 1.95} & \textbf{0.16} & \textbf{0.64} \\
PointPush2    & {\color[HTML]{0000FF}\textbf{0.18}} & {\color[HTML]{0000FF}\textbf{0.90}} & \textbf{0.15} & \textbf{0.76} & {\color[HTML]{656565} 0.18} & {\color[HTML]{656565} 1.69} & {\color[HTML]{656565} 0.12} & {\color[HTML]{656565} 1.19} & \textbf{0.14} & \textbf{0.56} & {\color[HTML]{656565} 0.08} & {\color[HTML]{656565} 1.19} & \textbf{0.05} & \textbf{0.27}
& {\color[HTML]{656565} -1.30} & {\color[HTML]{656565} 1.33} & {\color[HTML]{656565} 0.10} & {\color[HTML]{656565} 2.29} \\
CarButton1    & \textbf{-0.15} & \textbf{0.87} & {\color[HTML]{0000FF} \textbf{0.05}} & {\color[HTML]{0000FF} \textbf{0.51}} & {\color[HTML]{656565} 0.16} & {\color[HTML]{656565} 4.91} & {\color[HTML]{656565} 0.04} & {\color[HTML]{656565} 1.63} & {\color[HTML]{656565} 0.42} & {\color[HTML]{656565} 9.66} & {\color[HTML]{656565} -0.08} & {\color[HTML]{656565} 1.68} & \textbf{0.02} & \textbf{0.12}
& {\color[HTML]{656565} -0.43} & {\color[HTML]{656565} 9.15} & \textbf{-0.08} & \textbf{0.45} \\
CarButton2    & \textbf{-0.33} & \textbf{0.99} & \textbf{-0.01} & \textbf{0.71} & {\color[HTML]{656565} 0.08} & {\color[HTML]{656565} 5.87} & {\color[HTML]{656565} 0.06} & {\color[HTML]{656565} 2.13} & {\color[HTML]{656565} 0.37} & {\color[HTML]{656565} 12.51} & {\color[HTML]{656565} -0.07} & {\color[HTML]{656565} 1.59} & {\color[HTML]{0000FF} \textbf{0.01}} & {\color[HTML]{0000FF} \textbf{0.20}}
& {\color[HTML]{656565} -0.01} & {\color[HTML]{656565} 17.17} & {\color[HTML]{656565} -0.09} & {\color[HTML]{656565} 1.58} \\
CarCircle1    & {\color[HTML]{0000FF} \textbf{0.31}} & {\color[HTML]{0000FF} \textbf{0.39}} & \textbf{0.21} & \textbf{0.95} & {\color[HTML]{656565} 0.45} & {\color[HTML]{656565} 4.62} & {\color[HTML]{656565} 0.59} & {\color[HTML]{656565} 11.06} & {\color[HTML]{656565} -0.09} & {\color[HTML]{656565} 1.02} & {\color[HTML]{656565} 0.64} & {\color[HTML]{656565} 15.47} & {\color[HTML]{656565} 0.60} & {\color[HTML]{656565} 6.54}
& {\color[HTML]{656565} 0.41} & {\color[HTML]{656565} 4.16} & {\color[HTML]{656565} 0.46} & {\color[HTML]{656565} 4.41} \\
CarCircle2    & \textbf{0.48} & \textbf{0.94} & {\color[HTML]{656565} 0.54} & {\color[HTML]{656565} 3.38} & {\color[HTML]{656565} 0.45} & {\color[HTML]{656565} 6.24} & {\color[HTML]{656565} 0.53} & {\color[HTML]{656565} 8.35} & {\color[HTML]{0000FF} \textbf{0.50}} & {\color[HTML]{0000FF} \textbf{0.13}} & {\color[HTML]{656565} 0.64} & {\color[HTML]{656565} 18.15} & {\color[HTML]{656565} 0.45} & {\color[HTML]{656565} 1.46}
& {\color[HTML]{656565} 0.75} & {\color[HTML]{656565} 30.71} & {\color[HTML]{656565} 0.39} & {\color[HTML]{656565} 2.36} \\
CarGoal1      & {\color[HTML]{0000FF} \textbf{0.60}} & {\color[HTML]{0000FF} \textbf{0.44}} & \textbf{0.39} & \textbf{0.25} & {\color[HTML]{656565} 0.72} & {\color[HTML]{656565} 2.25} & {\color[HTML]{656565} 0.44} & {\color[HTML]{656565} 2.76} & {\color[HTML]{656565} 0.33} & {\color[HTML]{656565} 4.93} & {\color[HTML]{656565} 0.43} & {\color[HTML]{656565} 2.81} & \textbf{0.49} & \textbf{0.83}
& \textbf{-0.20} & \textbf{0.86} & {\color[HTML]{656565} 0.29} & {\color[HTML]{656565} 1.25} \\
CarGoal2      & {\color[HTML]{0000FF} \textbf{0.28}} & {\color[HTML]{0000FF} \textbf{0.64}} & \textbf{0.19} & \textbf{0.68} & {\color[HTML]{656565} 0.39} & {\color[HTML]{656565} 3.53} & {\color[HTML]{656565} 0.34} & {\color[HTML]{656565} 4.72} & {\color[HTML]{656565} 0.10} & {\color[HTML]{656565} 6.31} & {\color[HTML]{656565} 0.19} & {\color[HTML]{656565} 2.83} & \textbf{0.06} & \textbf{0.33}
& {\color[HTML]{656565} -0.51} & {\color[HTML]{656565} 2.02} & {\color[HTML]{656565} 0.09} & {\color[HTML]{656565} 1.05} \\
CarPush1      & {\color[HTML]{0000FF} \textbf{0.34}} & {\color[HTML]{0000FF} \textbf{0.65}} & \textbf{0.23} & \textbf{0.35} & \textbf{0.34} & \textbf{0.79} & {\color[HTML]{656565} 0.23} & {\color[HTML]{656565} 1.33} & \textbf{0.08} & \textbf{0.77} & {\color[HTML]{656565} 0.21} & {\color[HTML]{656565} 1.28} & \textbf{0.28} & \textbf{0.28}
& \textbf{-0.12} & \textbf{0.56} & \textbf{0.17} & \textbf{0.77} \\
CarPush2      & {\color[HTML]{0000FF} \textbf{0.18}} & {\color[HTML]{0000FF} \textbf{0.61}} & \textbf{0.10} & \textbf{0.91} & {\color[HTML]{656565} 0.11} & {\color[HTML]{656565} 2.33} & {\color[HTML]{656565} 0.10} & {\color[HTML]{656565} 2.78} & {\color[HTML]{656565} -0.03} & {\color[HTML]{656565} 10.00} & {\color[HTML]{656565} 0.10} & {\color[HTML]{656565} 4.55} & \textbf{0.14} & \textbf{0.89}
& {\color[HTML]{656565} -1.15} & {\color[HTML]{656565} 9.06} & {\color[HTML]{656565} 0.03} & {\color[HTML]{656565} 1.92} \\
SwimmerVelocity & {\color[HTML]{0000FF} \textbf{0.65}} & {\color[HTML]{0000FF} \textbf{0.58}} & \textbf{0.55} & \textbf{0.89} & \textbf{0.65} & \textbf{0.94} & {\color[HTML]{656565} 0.29} & {\color[HTML]{656565} 4.10} & {\color[HTML]{656565} 0.31} & {\color[HTML]{656565} 11.58} & {\color[HTML]{656565} 0.58} & {\color[HTML]{656565} 23.64} & \textbf{-0.04} & \textbf{0.00}
& \textbf{0.56} & \textbf{0.65} & {\color[HTML]{656565} 0.41} & {\color[HTML]{656565} 2.56} \\
HopperVelocity  & {\color[HTML]{0000FF} \textbf{0.88}} & {\color[HTML]{0000FF} \textbf{0.45}} & \textbf{0.58} & \textbf{0.45} & \textbf{0.76} & \textbf{0.97} & \textbf{0.12} & \textbf{0.97} & \textbf{0.57} & \textbf{0.00} & {\color[HTML]{656565} 0.23} & {\color[HTML]{656565} 1.44} & \textbf{0.19} & \textbf{0.51}
& \textbf{0.35} & \textbf{0.79} & \textbf{0.34} & \textbf{0.35} \\
HalfCheetahVelocity & {\color[HTML]{0000FF} \textbf{1.01}} & {\color[HTML]{0000FF} \textbf{0.03}} & \textbf{0.90} & \textbf{0.53} & \textbf{1.01} & \textbf{0.28} & {\color[HTML]{656565} 1.04} & {\color[HTML]{656565} 57.06} & {\color[HTML]{656565} 0.08} & {\color[HTML]{656565} 2.56} & \textbf{0.43} & \textbf{0.00} & \textbf{0.89} & \textbf{0.00}
& \textbf{0.65} & \textbf{0.25} & \textbf{0.87} & \textbf{0.76} \\
Walker2dVelocity & {\color[HTML]{0000FF} \textbf{0.83}} & {\color[HTML]{0000FF} \textbf{0.47}} & \textbf{0.81} & \textbf{0.31} & \textbf{0.83} & \textbf{0.97} & \textbf{0.81} & \textbf{0.37} & \textbf{0.31} & \textbf{0.65} & \textbf{0.09} & \textbf{0.84} & \textbf{0.23} & \textbf{0.83}
& {\color[HTML]{656565} 0.75} & {\color[HTML]{656565} 1.42} & \textbf{0.78} & \textbf{0.08} \\
AntVelocity     & {\color[HTML]{0000FF} \textbf{0.99}} & {\color[HTML]{0000FF} \textbf{0.78}} & \textbf{0.96} & \textbf{0.89} & \textbf{0.98} & \textbf{0.94} & {\color[HTML]{656565} 0.85} & {\color[HTML]{656565} 18.54} & \textbf{-1.01} & \textbf{0.00} & {\color[HTML]{656565} 1.00} & {\color[HTML]{656565} 10.29} & \textbf{0.89} & \textbf{0.00}
& \textbf{0.84} & \textbf{0.93} & \textbf{0.85} & \textbf{0.16} \\ \hline

\textbf{\begin{tabular}[c]{@{}c@{}}SafetyGym\\ Average\end{tabular}} 
& {\color[HTML]{0000FF} \textbf{0.42}} & {\color[HTML]{0000FF} \textbf{0.69}} & \textbf{0.36} & \textbf{0.74} & {\color[HTML]{656565} 0.51} & {\color[HTML]{656565} 2.73} & {\color[HTML]{656565} 0.40} & {\color[HTML]{656565} 6.70} & {\color[HTML]{656565} 0.24} & {\color[HTML]{656565} 3.45} & {\color[HTML]{656565} 0.34} & {\color[HTML]{656565} 6.45} & {\color[HTML]{656565} 0.31} & {\color[HTML]{656565} 2.01}
& {\color[HTML]{656565} 0.13} & {\color[HTML]{656565} 4.37} & {\color[HTML]{656565} 0.28} & {\color[HTML]{656565} 1.27} \\ \hline

BallRun    & \textbf{0.31} & \textbf{0.00} & \textbf{0.29} & \textbf{0.37} & \textbf{0.32} & \textbf{1.00} & \textbf{0.30} & \textbf{0.89} & {\color[HTML]{0000FF} \textbf{0.33}} & {\color[HTML]{0000FF} \textbf{0.00}} & \textbf{0.26} & \textbf{0.96} & \textbf{0.24} & \textbf{0.00}
& \textbf{0.27} & \textbf{0.00} & \textbf{0.17} & \textbf{0.02} \\
CarRun     & {\color[HTML]{0000FF} \textbf{0.99}} & {\color[HTML]{0000FF} \textbf{0.30}} & \textbf{0.98} & \textbf{0.34} & \textbf{0.99} & \textbf{0.78} & \textbf{0.98} & \textbf{0.13} & \textbf{0.98} & \textbf{0.23} & \textbf{0.95} & \textbf{0.54} & \textbf{0.76} & \textbf{0.00}
& \textbf{0.67} & \textbf{0.00} & \textbf{0.98} & \textbf{0.22} \\
DroneRun   & {\color[HTML]{0000FF} \textbf{0.60}} & {\color[HTML]{0000FF} \textbf{0.48}} & \textbf{0.57} & \textbf{0.00} & \textbf{0.59} & \textbf{0.80} & {\color[HTML]{656565} 0.68} & {\color[HTML]{656565} 4.47} & \textbf{0.46} & \textbf{0.00} & {\color[HTML]{656565} 0.57} & {\color[HTML]{656565} 6.67} & \textbf{0.31} & \textbf{0.16}
& {\color[HTML]{656565} 0.44} & {\color[HTML]{656565} 3.29} & {\color[HTML]{656565} 0.41} & {\color[HTML]{656565} 3.28} \\
AntRun     & {\color[HTML]{0000FF} \textbf{0.73}} & {\color[HTML]{0000FF} \textbf{0.82}} & \textbf{0.70} & \textbf{0.79} & \textbf{0.72} & \textbf{0.99} & \textbf{0.58} & \textbf{0.77} & \textbf{0.09} & \textbf{0.46} & \textbf{0.61} & \textbf{0.92} & \textbf{0.52} & \textbf{0.83}
& {\color[HTML]{656565} 0.61} & {\color[HTML]{656565} 1.70} & {\color[HTML]{656565} 0.54} & {\color[HTML]{656565} 1.03} \\
BallCircle & {\color[HTML]{0000FF} \textbf{0.68}} & {\color[HTML]{0000FF} \textbf{0.95}} & \textbf{0.55} & \textbf{0.08} & \textbf{0.68} & \textbf{0.97} & {\color[HTML]{656565} 0.68} & {\color[HTML]{656565} 1.57} & \textbf{0.71} & \textbf{0.30} & {\color[HTML]{656565} 0.64} & {\color[HTML]{656565} 3.28} & \textbf{0.36} & \textbf{0.00}
& \textbf{0.72} & \textbf{0.62} & \textbf{0.58} & \textbf{0.28} \\
CarCircle  & \textbf{0.67} & \textbf{0.90} & \textbf{0.55} & \textbf{0.43} & {\color[HTML]{0000FF} \textbf{0.73}} & {\color[HTML]{0000FF} \textbf{0.83}} & {\color[HTML]{656565} 0.46} & {\color[HTML]{656565} 1.42} & \textbf{0.73} & \textbf{0.89} & {\color[HTML]{656565} 0.46} & {\color[HTML]{656565} 2.78} & \textbf{0.42} & \textbf{0.16}
& {\color[HTML]{656565} 0.72} & {\color[HTML]{656565} 2.42} & \textbf{0.56} & \textbf{0.42} \\
DroneCircle & {\color[HTML]{0000FF} \textbf{0.60}} & {\color[HTML]{0000FF} \textbf{0.92}} & \textbf{0.57} & \textbf{0.56} & \textbf{0.58} & \textbf{0.96} & \textbf{0.52} & \textbf{0.98} & \textbf{-0.20} & \textbf{0.45} & \textbf{0.26} & \textbf{0.51} & \textbf{0.49} & \textbf{0.00}
& \textbf{0.36} & \textbf{0.24} & \textbf{0.51} & \textbf{0.49} \\
AntCircle  & \textbf{0.41} & \textbf{0.41} & {\color[HTML]{0000FF} \textbf{0.45}} & {\color[HTML]{0000FF} \textbf{0.98}} & {\color[HTML]{656565} 0.31} & {\color[HTML]{656565} 1.25} & {\color[HTML]{656565} 0.57} & {\color[HTML]{656565} 2.11} & \textbf{0.02} & \textbf{0.00} & {\color[HTML]{656565} 0.10} & {\color[HTML]{656565} 1.31} & \textbf{0.29} & \textbf{0.00}
& {\color[HTML]{656565} 0.35} & {\color[HTML]{656565} 1.04} & \textbf{0.34} & \textbf{0.01} \\ \hline

\textbf{\begin{tabular}[c]{@{}c@{}}BulletGym\\ Average\end{tabular}} 
& {\color[HTML]{0000FF} \textbf{0.62}} & {\color[HTML]{0000FF} \textbf{0.60}} & \textbf{0.58} & \textbf{0.44} & \textbf{0.61} & \textbf{0.95} & {\color[HTML]{656565} 0.60} & {\color[HTML]{656565} 1.54} & \textbf{0.39} & \textbf{0.29} & {\color[HTML]{656565} 0.48} & {\color[HTML]{656565} 2.12} & \textbf{0.42} & \textbf{0.14}
& {\color[HTML]{656565} 0.52} & {\color[HTML]{656565} 1.16} & \textbf{0.51} & \textbf{0.72} \\ \hline

easysparse  & {\color[HTML]{0000FF} \textbf{0.70}} & {\color[HTML]{0000FF} \textbf{0.98}} & \textbf{0.28} & \textbf{0.20} & \textbf{0.51} & \textbf{0.76} & \textbf{0.09} & \textbf{0.92} & \textbf{-0.05} & \textbf{0.19} & \textbf{0.07} & \textbf{0.86} & \textbf{0.44} & \textbf{0.28}
& \textbf{-0.01} & \textbf{0.09} & \textbf{0.19} & \textbf{0.10} \\
eastmean    & {\color[HTML]{0000FF} \textbf{0.67}} & {\color[HTML]{0000FF} \textbf{0.99}} & \textbf{0.49} & \textbf{0.06} & \textbf{0.52} & \textbf{0.99} & \textbf{0.08} & \textbf{0.70} & \textbf{-0.06} & \textbf{0.00} & \textbf{0.04} & \textbf{0.83} & \textbf{0.40} & \textbf{0.30}
& \textbf{-0.01} & \textbf{0.10} & \textbf{0.08} & \textbf{0.10} \\
easydense   & {\color[HTML]{0000FF} \textbf{0.65}} & {\color[HTML]{0000FF} \textbf{0.50}} & \textbf{0.59} & \textbf{0.01} & \textbf{0.47} & \textbf{0.87} & \textbf{0.04} & \textbf{0.80} & \textbf{-0.05} & \textbf{0.10} & {\color[HTML]{656565} 0.17} & {\color[HTML]{656565} 1.54} & \textbf{0.46} & \textbf{0.64}
& \textbf{-0.02} & \textbf{0.10} & \textbf{0.03} & \textbf{0.10} \\
mediumsparse & {\color[HTML]{0000FF} \textbf{0.97}} & {\color[HTML]{0000FF} \textbf{0.76}} & \textbf{0.50} & \textbf{0.10} & \textbf{0.52} & \textbf{0.03} & \textbf{0.92} & \textbf{0.42} & \textbf{-0.07} & \textbf{0.00} & \textbf{0.05} & \textbf{0.72} & \textbf{0.73} & \textbf{0.06}
& \textbf{-0.03} & \textbf{0.08} & \textbf{0.22} & \textbf{0.09} \\
mediummean  & {\color[HTML]{0000FF} \textbf{0.97}} & {\color[HTML]{0000FF} \textbf{0.66}} & \textbf{0.36} & \textbf{0.05} & \textbf{0.68} & \textbf{0.97} & \textbf{0.03} & \textbf{0.68} & \textbf{-0.06} & \textbf{0.00} & \textbf{0.09} & \textbf{0.77} & \textbf{0.52} & \textbf{0.01}
& \textbf{-0.03} & \textbf{0.10} & \textbf{0.42} & \textbf{0.08} \\
mediumdense & {\color[HTML]{0000FF} \textbf{0.97}} & {\color[HTML]{0000FF} \textbf{0.95}} & \textbf{0.25} & \textbf{0.10} & \textbf{0.25} & \textbf{0.10} & \textbf{0.94} & \textbf{0.29} & \textbf{-0.05} & \textbf{0.00} & \textbf{0.00} & \textbf{0.31} & \textbf{0.81} & \textbf{0.15}
& \textbf{-0.02} & \textbf{0.08} & \textbf{0.33} & \textbf{0.09} \\
hardsparse  & {\color[HTML]{0000FF} \textbf{0.44}} & {\color[HTML]{0000FF} \textbf{0.98}} & \textbf{0.24} & \textbf{0.00} & \textbf{0.37} & \textbf{0.48} & \textbf{0.47} & \textbf{0.80} & \textbf{-0.05} & \textbf{0.06} & {\color[HTML]{656565} 0.16} & {\color[HTML]{656565} 1.92} & \textbf{0.32} & \textbf{0.01}
& \textbf{-0.01} & \textbf{0.10} & \textbf{0.20} & \textbf{0.10} \\
hardmean    & {\color[HTML]{0000FF} \textbf{0.48}} & {\color[HTML]{0000FF} \textbf{0.94}} & \textbf{0.30} & \textbf{0.28} & \textbf{0.20} & \textbf{0.77} & \textbf{-0.01} & \textbf{0.44} & \textbf{-0.04} & \textbf{0.16} & \textbf{0.03} & \textbf{0.82} & \textbf{0.30} & \textbf{0.01}
& \textbf{-0.03} & \textbf{0.08} & \textbf{0.17} & \textbf{0.10} \\
harddense   & {\color[HTML]{0000FF} \textbf{0.50}} & {\color[HTML]{0000FF} \textbf{0.81}} & \textbf{0.27} & \textbf{0.39} & \textbf{0.24} & \textbf{0.16} & \textbf{0.05} & \textbf{0.74} & \textbf{-0.05} & \textbf{0.00} & \textbf{0.02} & \textbf{0.57} & \textbf{0.39} & \textbf{0.32}
& \textbf{-0.02} & \textbf{0.05} & \textbf{0.19} & \textbf{0.10} \\ \hline

\textbf{\begin{tabular}[c]{@{}c@{}}MetaDrive\\ Average\end{tabular}} 
& {\color[HTML]{0000FF} \textbf{0.71}} & {\color[HTML]{0000FF} \textbf{0.84}} & \textbf{0.36} & \textbf{0.13} & \textbf{0.42} & \textbf{0.57} & \textbf{0.29} & \textbf{0.64} & \textbf{-0.05} & \textbf{0.10} & \textbf{0.07} & \textbf{0.93} & \textbf{0.49} & \textbf{0.20}
& \textbf{-0.02} & \textbf{0.09} & \textbf{0.20} & \textbf{0.10} \\ \hline

\end{tabular}
}
\vspace{-.4cm}
\end{table*}

Theorem~\ref{thm:improve} establishes that, with high probability, Q-penalized training leads to improved expected reward and reduced expected cost relative to the baseline sequence-model training.
The complete proof and discussion are deferred to Appendix~\ref{sec:proofoft1}.
\section{Experiment}
\label{sec:Exp}
\textbf{Experimental Setups.} We conducted extensive evaluations on tasks from \textit{Safety-Gymnasium} \citep{ray2019benchmarking, ji2024omnisafe}, \textit{Bullet-Safety-Gym} \citep{gronauer2022bullet}, and \textit{MetaDrive} \citep{li2022metadrive}, utilizing the DSRL benchmark \citep{liu2023datasets} to assess the performance of QPT against state-of-the-art safe offline RL methods. 
The evaluation metrics used are normalized return and normalized cost, where a normalized cost below 1 is indicative of safety. In accordance with the DSRL benchmark, safety is prioritized as the primary evaluation criterion, with higher rewards pursued only after meeting safety requirements. To ensure fair comparisons, we set the cost limit for all tasks to 10.

\textbf{Baselines.} We compared our method with four types of baseline methods:
(1) \emph{Q-learning-based methods}: CPQ \citep{xu2022constraints}, BCQ-Lag \citep{fujimoto2019off, stooke2020responsive}, and CAPS \citep{chemingui2025constraint}, which augments a Q-learning backbone with an additional wrapper to enforce cost constraints;
(2) Distribution correction estimation: COptiDICE \citep{lee2022coptidice, lee2021optidice};
(3) Imitation learning: Behavior Cloning (BC-Safe) \citep{liu2023datasets}, which is trained exclusively on safe trajectories that satisfy safety constraints, FISOR \citep{zheng2024safe}, which leverages diffusion models for the development of safe policies, and OASIS \citep{yao2024oasis}, which uses diffusion-based data augmentation with BCQ-Lag as the underlying policy;
(4) Sequential modeling algorithms: CDT \citep{liu2023constrained}, which incorporates cost-to-go token in the training process.
The codebase for these baseline methods are sourced from \citet{liu2023datasets} and executed by us to ensure a fair comparison.
For evaluation, when the normalized cost is below 1, we select the configuration with the highest normalized reward. Otherwise, we prioritize minimizing normalized cost and report the corresponding normalized reward.

\textbf{Metrics.} We use the normalized cost return and the normalized reward return as the evaluation metric for comparison.
Denote $r_{max}(\gM)$ and $r_{min}(\gM)$ as the maximum empirical reward return and the minimum empirical reward return for task $\gM$.
The normalized reward is computed by:
\begin{equation}
    R_{\text{normalized}} = \frac{R_{\pi} - r_{min}(\gM)}{r_{max}(\gM) - r_{min}(\gM)} \times 100,
\end{equation}
where $R_{\pi}$ denotes the evaluated reward return of policy $\pi$.
While the normalized cost is computed by the ratio between the evaluated cost return $C_{\pi}$ and the target threshold $\kappa$:
\begin{equation}
    C_{\text{normalized}} = \frac{C_{\pi} + \epsilon}{\kappa + \epsilon},
\end{equation}
where $\epsilon$ is a positive number to ensure numerical stability if the threshold $\kappa = 0$. 
The agent is safe if $C_{\text{normalized}} \leq 1$.
Without otherwise statements, we will abbreviate ``normalized cost return" as ``cost" and ``normalized reward return" as ``reward" for simplicity.

\begin{table*}[!t]
\renewcommand{\arraystretch}{1.0}
  \centering
  \caption{Impact of different components. Average scores and standard deviations are reported over three random seeds for the \textit{harddense} task in the \textit{MetaDrive} setting and the \textit{HopperVelocity} setting.  ``Train with $Q^r$'' and ``Train with $Q^c$'' indicate whether the corresponding penalization in Equation \ref{eq:final_update} is applied. ``Data aug.'' refers to the use of data augmentation, while ``Inf. with ensemble" denotes ensemble applied at inference time.}
  \label{tab:component}
  \scalebox{0.9}{
  \begin{tabular}{C{0.5cm}C{1cm}C{1.2cm}C{1.2cm}C{1.5cm}C{2cm}C{2cm}C{2cm}C{2cm}}
    \hline
      Exp & Data aug. & Train with $Q^r$ & Train with $Q^c$ & Inf. with ensemble & harddense Reward & harddense Cost & HopperVelocity Reward & HopperVelocity Cost \\
    \hline
     1 &  & & & & $0.37 \pm 0.19$ & $1.00 \pm 0.08$ & $0.04 \pm 0.02$ & $1.49 \pm 0.12$ \\
     2 & \ding{51} & &  & & $0.40 \pm 0.05$  & $0.94 \pm 0.04$ & $0.54 \pm 0.03$ & $0.65 \pm 0.03$ \\ 
     3 & \ding{51} & \ding{51} & & & $0.48 \pm 0.08$ & $0.94 \pm 0.05$ & $0.85 \pm 0.05$ & $0.99 \pm 0.04$ \\
     4 & \ding{51} &  & \ding{51} & &  $0.43 \pm 0.06$ & $0.14 \pm 0.10$ & $0.15 \pm 0.02$ & $0.22 \pm 0.01$ \\
     5 & \ding{51} & \ding{51}  & \ding{51}  &  & $0.49 \pm 0.04$  & $0.84 \pm 0.02$ & $0.69 \pm 0.04$ & $0.51 \pm 0.02$ \\
     6 & \ding{51} & & & \ding{51} & $0.46 \pm 0.05$ & $0.91 \pm 0.06$ & $0.56 \pm 0.02$ & $0.60 \pm 0.04$ \\
     7 &  & \ding{51}  &  \ding{51} & \ding{51} & $0.50 \pm 0.01$& $0.93 \pm 0.02$ & $0.66 \pm 0.03$ & $0.56 \pm 0.03$   \\
     8 & \ding{51} & \ding{51}  &  \ding{51} & \ding{51} & $0.50 \pm 0.02$ & $0.81 \pm 0.03$ & $0.88 \pm 0.02$ & $0.45 \pm 0.04$ \\
    \hline
  \end{tabular}}
  \vspace{-0.4cm}
\end{table*}

\textbf{Main Results.} 
The evaluation results are summarized in Table \ref{tab:res}. QPT stands out as the only method that consistently achieves satisfactory safety performance across all tasks while also attaining the highest returns in most cases. This highlights its effectiveness in simultaneously ensuring safety and achieving high rewards.
In contrast, other methods exhibit significant limitations, either due to severe constraint violations or suboptimal returns.
Notably, BC-Safe, which is trained exclusively on safe trajectories, satisfies most safety requirements but demonstrates conservative performance with comparatively lower rewards.
Q-learning-based algorithms, including BCQ-Lag and CPQ, as well as the distribution correction estimation-based method, COptiDICE, show inconsistent performance. 
These methods tend to oscillate between overly conservative behavior and excessive risk-taking. For instance, CPQ achieves high rewards at the expense of significant safety violations in tasks such as \textit{CarGoal1} and \textit{CarGoal2}, while in \textit{MetaDrive} tasks, it achieves nearly zero cost but at the cost of extremely low rewards.
FISOR employs a diffusion-based architecture to maximize rewards within the largest safe region, thereby offering strong safety guarantees. However, it tends to sacrifice potential rewards by strictly adhering to the safe region, resulting in lower reward values while maintaining predominantly safe cost levels.
CDT, leveraging its advanced architecture and efficient data utilization, demonstrates more balanced performance. However, it still struggles with trade-offs between safety and utility in safe offline RL settings, particularly in \textit{SafetyGym} tasks, where it fails to meet safety requirements in most cases.
In contrast, QPT, which shares the same Transformer architecture as CDT, surpasses it by utilizing our novel framework. QPT achieves the highest returns while consistently satisfying safety requirements, underscoring the efficacy of our proposed method.

\begin{figure*}[!t]
    \centering
    \includegraphics[width=0.9\linewidth]{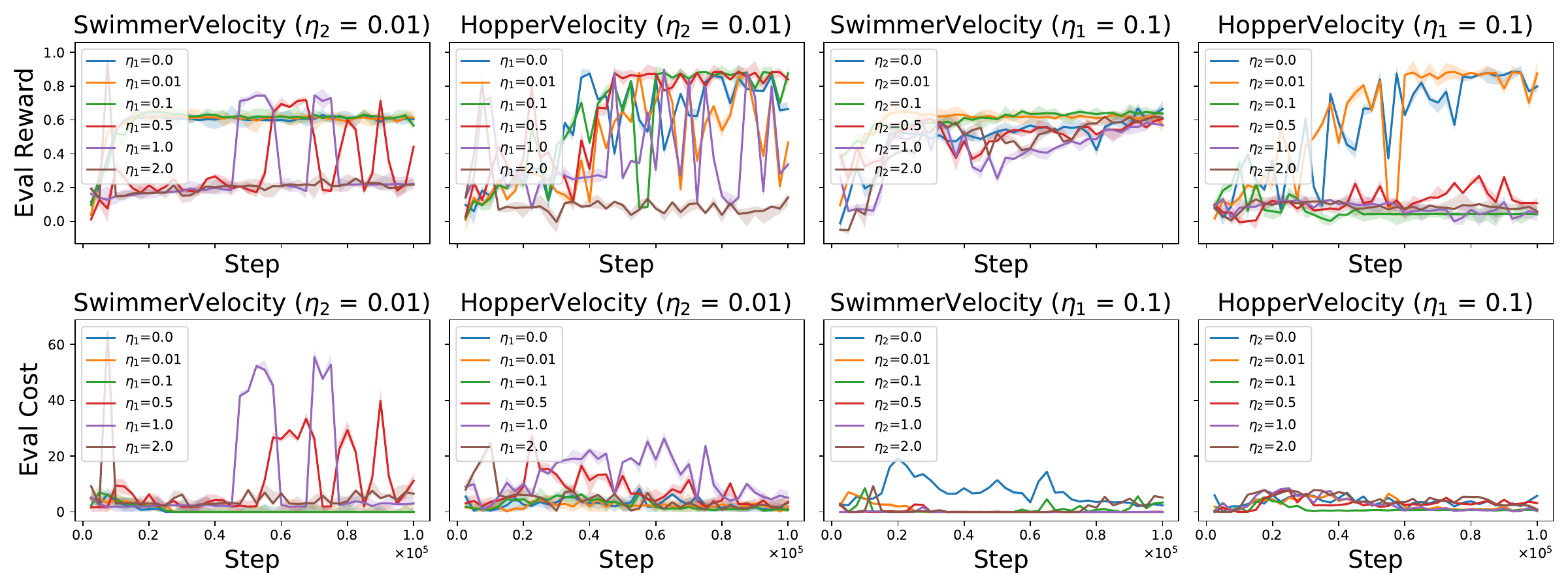}
    \vspace{-.3cm}
    \caption{Results of the impact of hyper-parameters $\eta_1$ and $\eta_2$. Each column is a task with one hyper-parameter stable. The x-axis is the training steps. The first row shows the evaluated normalized reward, and the second row shows the evaluated normalized cost. All plots are averaged among 3 random seeds and 20 trajectories for each seed. The solid line is the mean value, and the light shade represents the area within one standard deviation. }
    \label{fig:ablation_hyper}
    \vspace{-.2cm}
\end{figure*}

\subsection{Ablation}
\label{sec:ab}
\textbf{Role of Different Components.}
As detailed in Section \ref{sec:method}, our methodology incorporates four key components: data augmentation, reward Q-network, cost Q-network, and inference ensemble. Each component warrants individual analysis.
We evaluate these components on the \textit{harddense} dataset from the MetaDrive task and on the \textit{HopperVelocity} task, both selected for their challenging nature in achieving high rewards and for the substantial performance improvements that QPT demonstrates over baseline methods. The results are summarized in Table \ref{tab:component}.
Integrating the $Q^r$ and $Q^c$ networks substantially enhances performance, as evidenced by comparisons between Exp 2 vs. 3 and Exp 2 vs. 4, where the added Q-learning penalization leads to notable improvements in reward and cost metrics.
Furthermore, incorporating an ensemble of learned Q-networks further boosts performance, as shown by comparisons between Exp 2 vs. 6 and Exp 5 vs. 8.
Data augmentation also improves cost performance by ``stitching'' additional safe trajectories into the training dataset, as demonstrated by Exp 1 vs. 2 and Exp 7 vs. 8.
These findings highlight the effectiveness of our method in addressing the challenges of safe offline RL.

\textbf{Hyper-parameters.}
This ablation introduces the hyper-parameters $\eta_1$ and $\eta_2$, as defined in Equation \ref{eq:final_update}, which regulate the influence of two additional loss components.
To evaluate their effects, we conducted an ablation study on two tasks: \textit{SwimmerVelocity} and \textit{HopperVelocity}.
As illustrated in Figure \ref{fig:ablation_hyper}, when $\eta_2$ is held constant and $\eta_1$ is gradually increased, the normalized reward rises within a specific range. 
For instance, in the \textit{HopperVelocity} task, the reward increases consistently when $\eta_1 \leq 1$. However, beyond this point (e.g., $\eta_1 = 2$), the reward decreases sharply with no further performance gains.
Conversely, increasing $\eta_1$ also leads to a corresponding increase in the normalized cost within a certain range, indicating that the policy faces challenges in stitching trajectories with higher associated costs.
Similarly, when $\eta_1$ is kept constant and $\eta_2$ is increased, the normalized reward decreases progressively, accompanied by a reduction in the normalized cost within a certain range. This observation suggests that a larger $\eta_2$ can contribute to a safer policy, albeit with a trade-off in reward performance within specific bounds.

\begin{figure*}[!t]
    \centering
    \includegraphics[width=0.9\linewidth]{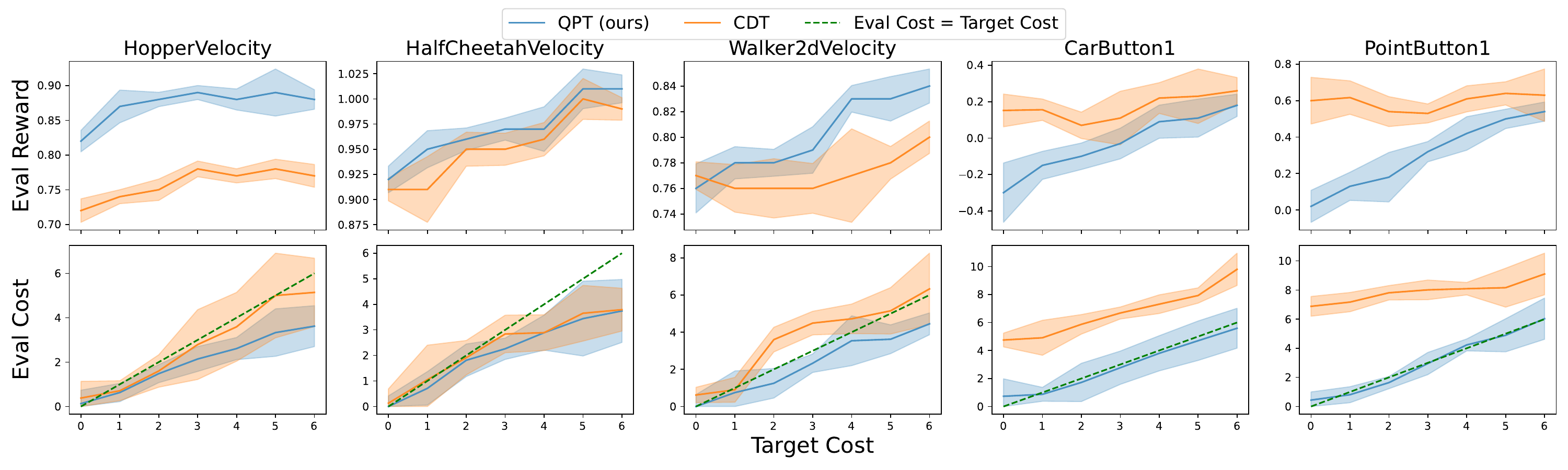}
    \vspace{-.3cm}
    \caption{Results of zero-shot adaption to different cost returns. Each column is a task. The x-axis is the target cost return. The first row shows the evaluated normalized reward, and the second row shows the evaluated normalized cost under different target costs. All plots are averaged among 3 random seeds and 20 trajectories for each seed. The solid line is the mean value, and the light shade represents the area within one standard deviation.}
    \label{fig:zeroshot}
    \vspace{-.5cm}
\end{figure*}

\textbf{Zero-shot Adaptation.}
One significant advantage of the Transformer-based policy is its capability for zero-shot adaptation to varying cost thresholds \citep{liu2023constrained}.
In contrast, the Q-learning-based baselines introduced earlier lack this capability, as they require a fixed, pre-defined threshold to solve constrained optimization problems. 
Adapting these methods to new constraint conditions necessitates re-training, which limits their flexibility.
Consequently, we primarily compare our method with CDT, which also supports zero-shot adaptation.
In this evaluation, each cost-return threshold is treated as a distinct task, with corresponding adjustments made to the target reward.
The results are presented in Figure \ref{fig:zeroshot}.
Both methods demonstrate improved performance when conditioned on a higher cost threshold, highlighting the zero-shot adaptation capability of sequence modeling approaches.
Furthermore, our method consistently outperforms CDT in scenarios where both methods satisfy the constraint, achieving lower costs than the specified threshold (as shown in the left three plots of Figure \ref{fig:zeroshot}).
Although CDT achieves higher rewards than our method in the CarButton1 and PointButton1 environments, it does so at the expense of greater safety violations. In contrast, our method adheres to the cost threshold, underscoring its effectiveness in maintaining safety while delivering competitive performance.
We attribute QPT's stronger zero-shot performance to the way constraint information enters training.
In offline datasets, cost-to-go estimates are often noisy surrogates of the trajectory-level expected cost, and this mismatch can be amplified under distribution shift \citep{QDT}. As a result, directly conditioning on or optimizing against these tokens may miscalibrate the learned reward-cost trade-off.
By contrast, QPT leverages Q-learning guidance to provide a more informative training signal that better reflects the latent cost-to-go distribution, thereby enabling more reliable cost control and a superior reward-cost balance.

\vspace{-.1cm}
\section{Conclusion}
\vspace{-.1cm}
\label{sec:Con}

This paper studies safe offline RL through the lens of trilogy optimization and proposes Q-learning Penalized Transformer policy (QPT).
QPT couples a conditional Transformer policy with reward/cost Q-functions via a Q-shaped objective, enabling behavior-regularized learning that favors high return under low constraint violation.
We provide a theoretical analysis in a stylized setting and evaluate QPT on 38 DSRL tasks, where it consistently outperforms strong safe offline RL baselines and exhibits robust zero-shot adaptation to different constraint thresholds.

\section*{Impact Statement}

This paper presents work regarding the safe policy learned from the offline dataset. While there are many potential societal consequences of our work, we believe that none require specific emphasis in this context.


\bibliography{example_paper}
\bibliographystyle{icml2026}

\newpage
\appendix
\onecolumn

\paragraph{Limitations.} The theoretical guarantees of our algorithms rely on a near-deterministic environment, an assumption that does not always hold in real-world deep learning models. Moreover, current evaluations rely on limited offline data, potentially constraining performance. Extending QPT to a safe offline-to-online RL framework presents a promising direction to enhance performance through online interactions while maintaining safety.

\section{Proof of Theorem \ref{thm:improve}}
\label{sec:proofoft1}

First, we give the following Lemma.

\begin{lemma}[A stylized guarantee for return-conditioned sequence modeling \citep{RCSL}]\label{thm:infinite}
Consider an MDP, behavior $ \beta$ and conditioning function $f^r$. Let $J^r(\pi) = \E_{\tau \sim \pi} [g^r(\tau)]$, where $g^r(\tau) = \sum_{t=1}^\gH r_t$. Assume the following:
\begin{enumerate}
    \item Return coverage: $P_{\beta} (g^r(\tau) = f^r(\rvs_1) |\rvs_1 ) \geq \alpha_{f^r}$ for all initial states $ \rvs_1$.
    \item Near determinism: $ P(r \neq \gR(\rvs, \rva) \text{ or } s' \neq \gT(\rvs, \rva) | \rvs,\rva ) \leq \epsilon$ at all $ s, a $ for some functions $ \gT$ and $ \gR $. Note that this does not constrain the stochasticity of the initial state.
    \item Consistency of $f^r$: $f^r(\rvs) = f^r(\rvs') + r$ for all $\rvs$. \footnote{Note this can be exactly enforced (as in prior work) by augmenting the state space to include the cumulative reward observed so far.}
\end{enumerate}

Then
\begin{align}
    J^r(\pi^*) - J^r(\pi) \leq \epsilon\left( \frac{1}{\alpha_{f^r}} + 3\right)\gH^2,
\end{align}
where $\pi$ is the policy induced by training with Eq.~\ref{eq:DTLoss}, and $\pi^*$ denotes an optimal policy in this stylized setting, and $\gH$ is the horizon length of episode.
Moreover, there exist problems where the bound is tight up to constant factors.
\end{lemma}

\begin{corollary}
    \label{thm:cost_infi}
    Under assumptions analogous to those in Lemma \ref{thm:infinite} for the cost function, specifically: $P_{\beta} (g^c(\tau) = f^c(\rvs_1) |\rvs_1 ) \geq \alpha_{f^c}$, $P(c \neq \gC(\rvs, \rva) \text{ or } s' \neq \gT(\rvs, \rva) | \rvs,\rva ) \leq \epsilon$, $f^c(\rvs) = f^c(\rvs') + c$.
    Let $J^c(\pi) = \E_{\tau \sim \pi} [g^c(\tau)]$, the following bound holds:
    \begin{align}
        J^c(\pi) - J^c(\pi^*) \leq \epsilon\left( \frac{1}{\alpha_{f^c}} + 3\right)\gH^2.
    \end{align}
\end{corollary}
The proof follows the same argument as Lemma~\ref{thm:infinite}.

Based on Lemma \ref{thm:infinite} and Corollary \ref{thm:cost_infi}, we now give the proof of Theorem \ref{thm:improve}.

\begin{proof}
    We prove the bounds for reward and cost separately.

    \textbf{Reward Bound.} For the reward, we begin with:
    \begin{align}
        &\E_{\tau \sim \pi^*}[g^r(\tau)] - \E_{\tau \sim \hat{\pi}} [g^r(\tau)] \\
        &=\E_{\tau \sim \pi^*}[g^r(\tau)] - \E_{\tau \sim \pi} [g^r(\tau)] + \E_{\tau \sim \pi} [g^r(\tau)] - \E_{\tau \sim \hat{\pi}} [g^r(\tau)] \\
        &=J^r(\pi^*) - J^r(\pi) + J^r(\pi) - \E_{\tau \sim \hat{\pi}} [g^r(\tau)] \\
        &\leq \epsilon\left( \frac{1}{\alpha_f} + 3\right)\gH^2 + J^r(\pi) - \E_{\tau \sim \hat{\pi}} [g^r(\tau)]. \label{eq:first}
    \end{align}
    Next, for the second term in Equation \ref{eq:first}:
    \begin{align}
        &J^r(\pi) - \E_{\tau \sim \hat{\pi}} [g^r(\tau)] \\
        &= \E_{\tau \sim \pi} [g^r(\tau)] - \E_{\tau \sim \hat{\pi}} [g^r(\tau)] \\
        &= \E_{\tau \sim \pi} [\sum_{t=1}^\gH (r_t)] - \E_{\tau \sim \hat{\pi}} [\sum_{t=1}^\gH (r_t)] \\
        &= \E_{\rvs_1} \sum_{t=1}^\gH (P^r_t \cdot r_t) - \E_{\rvs_1} \sum_{t=1}^\gH (\hat{P}^r_t \cdot r_t), \\
    \end{align}
    where $P^r_t$ and $\hat{P}^r_t$ represent the probabilities of selecting the maximum-reward actions under policies derived from Equation \ref{eq:DTLoss} and Equation \ref{eq:final_update}, respectively.
    Since rewards are binary and by the condition $P\{\hat{P}^r_i - P^r_i  \geq \sigma_r, \forall~i \} \geq 1 - \delta_r$, we have:
    \begin{align}
        &\E_{\rvs_1} \sum_{t=1}^\gH (P^r_t \cdot r_t) - \E_{\rvs_1} \sum_{t=1}^\gH (\hat{P}^r_t \cdot r_t) \\
        &= \E_{\rvs_1} \sum_{t=1}^\gH [(P^r_t - \hat{P}^r_t)r_t] \\
        &\leq \E_{\rvs_1} \sum_{t=1}^\gH (-\sigma_r) \cdot r_t \\
        & \leq -\gH \sigma_r. \label{eq:second}
    \end{align}
    Substituting Equation \ref{eq:second} into Equation \ref{eq:first}, we get:
    \begin{align}
        &\E_{\tau \sim \pi^*}[g^r(\tau)] - \E_{\tau \sim \hat{\pi}} [g^r(\tau)] \\
        &\leq \epsilon\left( \frac{1}{\alpha_f} + 3\right)\gH^2 - \gH \sigma_r. \label{eq:proofres1}
    \end{align}

    \textbf{Cost Bound.} Similarly,
    \begin{align}
    \E_{\tau \sim \hat{\pi}}[g^c(\tau)] - \E_{\tau \sim \pi^*}[g^c(\tau)]
    &=
    \bigl(J^c(\hat{\pi}) - J^c(\pi)\bigr)
    +
    \bigl(J^c(\pi) - J^c(\pi^*)\bigr) \nonumber\\
    &\le
    \bigl(J^c(\hat{\pi}) - J^c(\pi)\bigr)
    +
    \epsilon\left(\frac{1}{\alpha_{f^c}}+3\right)\gH^2. \label{eq:cost_first}
    \end{align}
    Moreover,
    \begin{align}
    J^c(\hat{\pi}) - J^c(\pi)
    &=
    \E_{\rvs_1}\sum_{t=1}^{\gH}(\hat P_t^c - P_t^c)\,c_t \nonumber\\
    &\le -\gH\sigma_c, \label{eq:cost_second}
    \end{align}
    where we used the binary cost and the event $\{\hat P_t^c - P_t^c \ge \sigma_c,\forall t\}$.
    Combining \eqref{eq:cost_first}--\eqref{eq:cost_second} yields
    \[
    \E_{\tau \sim \hat{\pi}}[g^c(\tau)] - \E_{\tau \sim \pi^*}[g^c(\tau)]
    \le
    \epsilon\left(\frac{1}{\alpha_{f^c}}+3\right)\gH^2 - \gH\sigma_c.
    \]
    where $P^c_t$ and $\hat{P}^c_t$ represent the probabilities of selecting minimum-cost actions under policies derived from Equation \ref{eq:DTLoss} and Equation \ref{eq:final_update}, respectively.
\end{proof}

\begin{remark}[Scope of the analysis]
The above results are derived in a stylized setting (binary signals, near-determinism, and consistent conditioning) to isolate the effect of Q-penalized updates relative to the baseline objective in Equation \ref{eq:DTLoss}.
They should be viewed as an explanatory guarantee rather than a full characterization of deep offline RL with function approximation.
\end{remark}

\begin{remark}[Interpretation]
Compared with Lemma~\ref{thm:infinite} (and its cost analogue), Theorem~\ref{thm:improve} introduces additional improvement terms $-\gH\sigma_r$ and $-\gH\sigma_c$, capturing the gain from Q-penalized training in increasing the probability of selecting reward-maximizing and cost-minimizing actions, respectively.
\end{remark}

\begin{corollary}
    If $\alpha_{f^r} > 0, \epsilon = 0$, and $f^r(s_1) = V^{r*}(s_1)$ for all initial states $s_1$, then $J^r(\pi^*) = J^r(\pi) = J^r(\hat{\pi})$ under $\hat{P}_i^r=P_i^r$ and $\sigma_r=0$. Analogously, if $\alpha_{f^c}>0$, $\epsilon=0$, and $f^c(s_1)=V^{c*}(s_1)$ for all $s_1$, then $J^c(\pi^*)=J^c(\pi)=J^c(\hat{\pi})$.
\end{corollary}

\begin{remark}[Reward--cost trade-off]
The reward-side and cost-side conditions are stated and analyzed separately.
In general constrained control problems, optimizing reward and minimizing cost may not be simultaneously achievable under a single conditioning specification.
QPT addresses this trade-off operationally by combining reward/cost terms in Equation \ref{eq:final_update} and enforcing a user-specified cost threshold at inference.
\end{remark}

\section{Algorithm Details}
\subsection{Algorithm Pseudocode}

The detailed pipeline of QPT is summarized in Algorithm  \ref{alg:QPT}.

\begin{algorithm*}[tb]
   \caption{QPT: Q-learning Penalized Transformer}
   \label{alg:QPT}
\begin{algorithmic}
   \STATE {\bfseries Input:} Sequence horizon $K$, offline datasets $\gD$, coefficient $\rho$, a set of candidate pairs of return-to-go and cost-to-go  $\{(\hat{r}_0^0, \hat{c}_0^0), (\hat{r}_0^1, \hat{c}_0^1), \dots, (\hat{r}_0^m, \hat{c}_0^m) \}$.
   \STATE Initialize policy network $\pi_{\theta}$, reward Q-networks $Q^r_{\phi_1}, Q^r_{\phi_2}$, cost Q-networks $Q^c_{\psi_1}, Q^c_{\psi_2}$, and target networks $\pi_{\theta'}, Q^r_{\phi_1'}, Q^r_{\phi_2'}, Q^c_{\psi_1'}$ and $Q^c_{\psi_2'}$.
   \STATE Update the dataset $\gD$ with data augmentation technique based on Equation \ref{eq:dataaug}.
   \STATE {\color{gray}\te{// Train the QPT}}
   \FOR{$t=1$ {\bfseries to} $\gH$}
   \STATE Sample sequence transition mini-batch $\gB = \{ (\hat{r}_j, \hat{c}_j, \rvs_j, \rva_j, r_j, c_j)_{j=t}^{t+K}, \} \sim \gD$.
   \STATE {\color{gray}\te{// Reward Q-network and cost Q-network learning}}
   \STATE Sample $\hat{\rva}_{t+K} \sim \pi_{\theta'} (\hat{\rva}_{t+K} | \hat{r}_{t : t+K}, \hat{c}_{t : t+K}, \rvs_{t : t+K}, \rva_{t : t+K-1})$.
   \STATE Update $Q^r_{\phi_1}$ and $Q^r_{\phi_2}$ by Equation \ref{eq:Q_update}, update $Q^c_{\psi_1}$ and $Q^c_{\psi_2}$ by Equation \ref{eq:Qc_update}.
   \STATE {\color{gray}\te{// Policy learning}}
   \FOR{$i=1$ {\bfseries to} $K$}
   \STATE Sample $\hat{\rva}_{t+i} \sim \pi_{\theta} (\hat{\rva}_{t+i} | \hat{r}_{t:t+i}, \hat{c}_{t:t+i}, \rvs_{t:t+i}, \rva_{t:t+i-1})$ in an auto-regressive way.
   \ENDFOR.
   \STATE Update policy by minimizing Equation 
   \ref{eq:final_update}.
   \STATE $\theta' = \rho \theta' + (1-\rho)\theta, \phi'_i = \rho \phi_i' + (1-\rho) \phi_i$, $ \psi'_i = \rho \psi_i' + (1-\rho) \psi_i$ for $i = \{1 ,2\}$.
   \ENDFOR.
   \STATE {\color{gray}\te{// Inference with QPT}}
   \STATE Given multiple pairs of target return-to-go and target cost-to-go choice $(\hat{r}^j, \hat{c}^j)^{j=1:m}_0$ and initial state $s_0$.
   \REPEAT
   \STATE Sample multiple actions with different return-to-go $\hat{\rva}^j_t = \pi_{\theta} (\hat{\rva}^j_t | \hat{r}^j_{t-K+1:t}, \hat{c}^j_{t-K+1:t}, \rvs_{t-K+1:t}, \rva_{t-K+1:t-1})$ for $j={1, \dots, m}$.
   \STATE Compute Q networks with candidate state-action pair $(\rvs_t, \hat{\rva}_t^j)$ for $j={1, \dots, m}$.
   \STATE Sample the action $\rva_t$ from action set $\{ \hat{\rva}_t^j \}_{j=1}^m$ with Equation \ref{eq:inf} and Equation \ref{eq:inf_c}.
   \STATE Execute the action $\rva_t$ and collect the reward $r_t$, cost $c_t$ and next state $\rvs_{t+1}$.
   \STATE Update current return-to-go $\hat{r}^j_{t+1} = \hat{r}^j_{t} - r_t$, cost-to-go $\hat{c}^j_{t+1} = \hat{c}^j_{t} - c_t$ for $j={1, \dots, m}$.
   \UNTIL{$Done$ is $true$.}
\end{algorithmic}
\end{algorithm*}

\subsection{Data Augmentation}
\label{sec:app_DA}
The intuition is to relabel the associated Pareto trajectory's reward and cost returns, such that the agent can learn to imitate the behavior of the most rewarding and safe trajectory $\tau^*$ when the desired return $(\rho, \kappa)$ is infeasible, i.e., $\rho > \text{RF}(\kappa, \gD)$. 
The Reward Frontier (RF) value is defined by the maximum reward with cost $\kappa \in \mathbb{C}$, where $\mathbb{C} := \{ C(\tau) : \tau \in \gD \}$ is the set of all the possible episodic cost in $\gD$:
\begin{align}
    \text{RF}(\kappa, \gD) = \max_{\tau \in \gD} R(\tau), ~~s.t.~~C(\tau) = \kappa.
\end{align}

The augmentation procedure is detailed in Algorithm~\ref{alg:DA} \citep{liu2023constrained}.
Figure~\ref{fig:dataaug} provides an illustrative example: arrows map Pareto-optimal trajectories to their corresponding augmented return-cost pairs.

\begin{figure}
    \centering
    \includegraphics[width=1.0\linewidth]{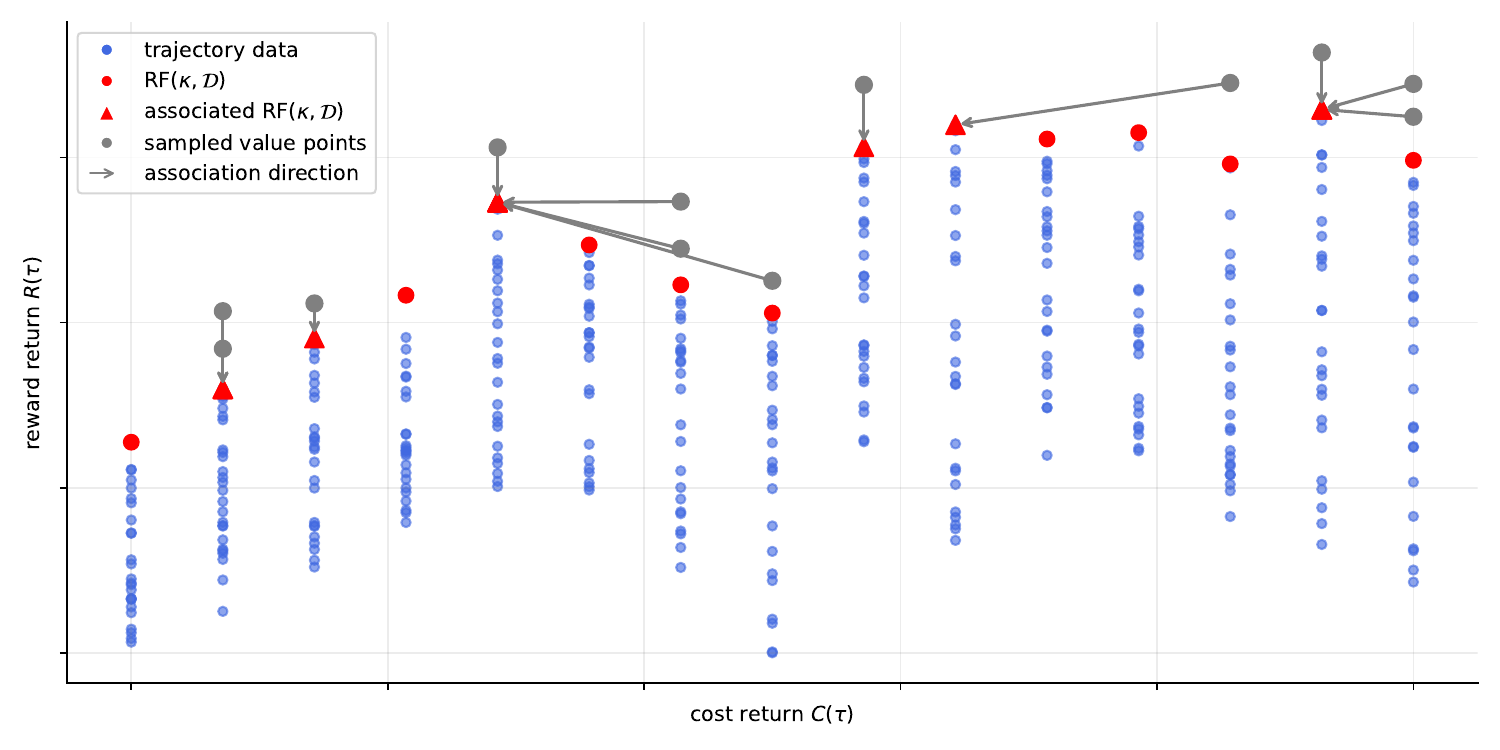}
    \vspace{-.8cm}
    \caption{Illustrative example of the data-augmentation procedure \citep{liu2023constrained}.}
    \label{fig:dataaug}
\end{figure}

\begin{algorithm}[h]
\caption{Data Augmentation via Relabeling}
{\bfseries Input:} \raggedright dataset $\gD$, samples $N$, reward sample max $r_{max}$ \par
{\bfseries Output:} \raggedright augmented trajectory dataset $\gD$ \par
\begin{algorithmic}[1] 
\STATE $c_{min} \leftarrow \min_{\tau\sim\gD} C(\tau)$, $c_{max} \leftarrow \max_{\tau\sim\gD} C(\tau)$
\FOR{$i=1,..., N$}
\STATE $\triangleright$ \textit{sample a cost return}
\STATE $\kappa_i \sim \text{Uniform}(c_{min}, c_{max})$
\STATE $\triangleright$ \textit{sample a reward return above the RF value}
\STATE $\rho_i \sim \text{Uniform}(\text{RF}(\kappa_i, \gD), r_{max})$
\STATE $\triangleright$ \textit{find the closest and safe Pareto trajectory}
\STATE $\tau^*_i \leftarrow \arg\max_{\tau\sim\gD} R(\tau), s.t. \quad C(\tau) \leq \kappa_i$
\STATE $\triangleright$ \textit{relabel the reward and cost return}
\STATE $\hat{\tau}_i \leftarrow \{\hat{r}^*_i + \rho_i - R(\tau^*_i),\hat{c}^*_i+ \kappa_i - C(\tau^*_i),\mathbf{s}^*_i,\mathbf{a}^*_i\}$
\STATE $\triangleright$ \textit{append the trajectory to the dataset}
\STATE $\gD \leftarrow \gD \cup \{\hat{\tau}_i\}$
\ENDFOR
\end{algorithmic} \label{alg:DA}
\end{algorithm}

\subsection{Ensemble}
\label{sec:app_ens}

In this section, we highlight the detailed ensemble process. 
During the training phase, RTG and CTG values are derived directly from the trajectory data within the dataset. Specifically, these values are computed as the cumulative discounted rewards and costs from each state to the terminal state along the observed trajectories, thereby preserving the ground-truth signal from the environment.
For inference, we utilize the default RTG and CTG pairs established in the DSRL benchmark as our baseline. To generate candidate pairs, we perturb the RTG values by introducing random noises while maintaining constant CTG values across all candidates. This asymmetric perturbation strategy is theoretically motivated: our objective is to maximize expected returns while adhering to a fixed cost constraint. By holding CTG constant while exploring a diverse range of RTG values, we effectively search the action-value landscape for optimal policies that maximize reward within the predetermined cost threshold.

The ensemble inference mechanism represents a computationally efficient approach to action selection that leverages our learned Q-networks to identify optimal actions from multiple candidates. This process operates exclusively during the inference phase and involves the following structured procedure:
First, we generate multiple return-to-go and cost-to-go conditioning pairs by introducing controlled stochastic perturbations to a default reference pair. This creates a diverse set of conditioning signals that explore different regions of the reward-safety trade-off space. Rather than processing these pairs sequentially, we exploit the parallelization capabilities of modern GPU architectures by batching all candidate pairs into a single forward pass through our trained Transformer model. This parallelization technique ensures that the computational overhead remains minimal compared to evaluating a single conditioning pair.
Once the model generates actions corresponding to each conditioning pair, we employ our learned Q-networks ($Q^r$ and $Q^c$) as evaluation metrics to select the optimal action according to specified criteria. Depending on the deployment context, these criteria may prioritize reward maximization subject to hard safety constraints, or implement a parameterized trade-off between reward and safety considerations.
By dynamically evaluating multiple conditioning pairs during inference, our method effectively automates this hyperparameter selection process, reducing the need for exhaustive offline tuning while potentially discovering superior action candidates that might be overlooked in a single-sample approach.

\section{Experiment Details}

\subsection{Environment descriptions}
The environments designed for evaluating safe offline RL methods are based on different simulators, each tailored to specific tasks and agent types. Figure \ref{fig:task-vis} visualize some representative tasks of these environments.

\begin{figure}[!h]
    \centering
    \includegraphics[width=0.9\linewidth]{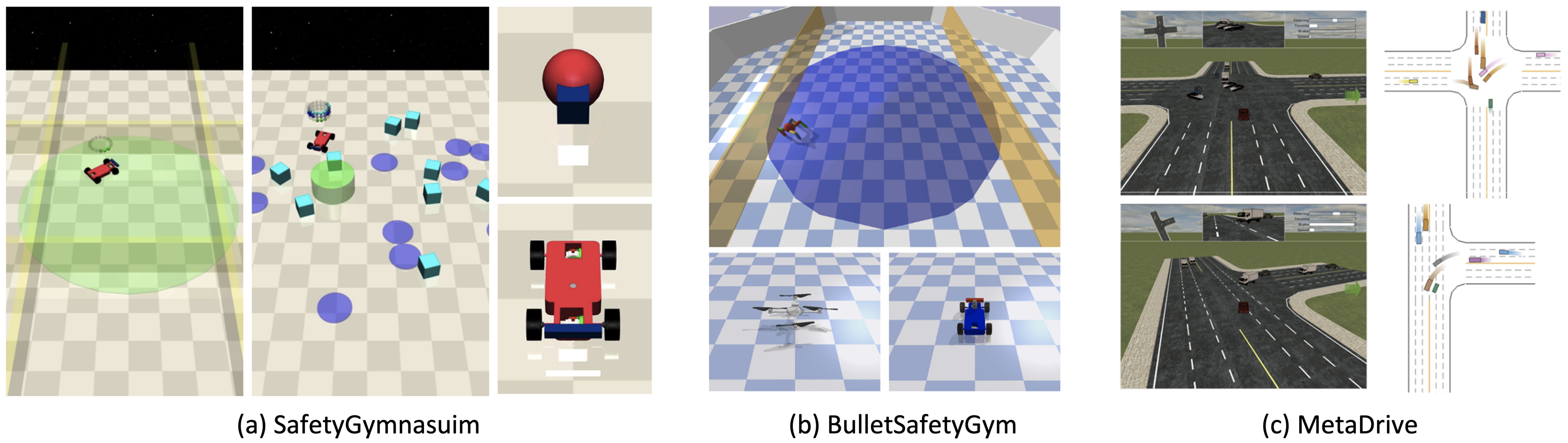}
    \caption{Visualization of the simulation environments and representative tasks \citep{liu2023datasets}.}
    \label{fig:task-vis}
\end{figure}

\textbf{Safety-Gymnasium} \citep{ray2019benchmarking, ji2024omnisafe}:
Built on the Mujoco physics simulator, Safety-Gymnasium provides safety-critical environments with diverse tasks. The Car agent engages in tasks such as Button, Push, and Goal, each available in two difficulty levels. These tasks require navigating hazards while completing objectives. For example, in Goal, the agent moves toward randomly reset goal positions upon completion. In Push, it moves a box to dynamic goal locations, while Button involves pressing scattered goal buttons. Additional velocity-constrained tasks are included for agents such as Ant, HalfCheetah, and Swimmer. The Velocity task challenges agents to coordinate leg movements to move forward, while Run requires navigating from a random direction and speed to a designated endpoint. The Circle task rewards agents for following a circular path while avoiding hazardous zones. Tasks are named by combining agent, task, and difficulty level (e.g., CarPush1), reflecting complexity and objectives. 

\textbf{Bullet-Safety-Gym} \citep{gronauer2022bullet}: Developed with the PyBullet physics simulator, this suite includes four agent types--Ball, Car, Drone, and Ant--and two primary tasks: Circle and Run. In Run, agents navigate corridors bounded by safety lines, incurring penalties for crossing them or exceeding speed limits. In Circle, agents move clockwise along a circular path, earning rewards for higher speeds near the boundary and penalties for straying outside the safety zone. These environments focus on safety evaluation with shorter, more straightforward tasks compared to Safety-Gymnasium. 

\textbf{MetaDrive} \citep{li2022metadrive}: MetaDrive is a self-driving simulation environment based on the Panda3D game engine, offering realistic driving conditions with varying road complexity (easy, medium, hard) and traffic density (sparse, medium, dense). Tasks are named by their road and vehicle conditions. This environment enables testing offline RL algorithms in scenarios that closely mimic real-world driving challenges.

An overview of these environments and tasks is presented in Table \ref{tab:environment-overview}. Each environment presents unique challenges for safe offline RL evaluation, from self-driving simulations to hazard-avoidance tasks, offering varied complexities and objectives for testing algorithm robustness.

\begin{table}[ht]
\centering
\caption{ Overview of the safe RL benchmarks and tasks for dataset collection \citep{liu2023datasets}.} 
\label{tab:environment-overview}
\renewcommand{\arraystretch}{1.1}
\resizebox{1.\linewidth}{!}{
\begin{tabular}{|c|c|c|c|c|c|c|}
\hline
Benchmarks  & Backends & Environments  & Agents  & \begin{tabular}[c]{@{}c@{}}Difficulty\\ Levels\end{tabular} & \begin{tabular}[c]{@{}c@{}}Total\\ Tasks\end{tabular} & \begin{tabular}[c]{@{}c@{}}Dataset\\ Trajectories\end{tabular} \\ \hline
\multirow{2}{*}{SafetyGymnasium} & \multirow{2}{*}{Mujoco} & \begin{tabular}[c]{@{}c@{}}Goal, Button,\\ Push, Circle\end{tabular} & Point, Car            & 2   & 16    & 40310 \\ \cline{3-7} 
 & & Velocity  & \begin{tabular}[c]{@{}c@{}}Ant, HalfCheetah, Hopper,\\ Swimmer, Walker2d\end{tabular} & 1   & 5  & 11399 \\ \hline
BulletSafetyGym & PyBullet  & Run, Circle  & Ball, Car, Drone, Ant & 1  & 8  & 14498  \\ \hline
MetaDrive  & Panda3D & Driving   & Vehicle  & 3  & 9  & 9000  \\ \hline
\end{tabular}
}
\end{table}

\subsection{Hyperparameters}

The reward and cost Q-networks used across all tasks consist of four linear layers, each employing Mish activation functions for non-linearity. To ensure a fair comparison between QPT and the baseline methods, we use a consistent setup of $10^5$ gradient steps (except for the \textit{MetaDrive} tasks) and a rollout length equal to the maximum episode length for all experiments. 
A comprehensive list of hyperparameters utilized in the experiments is provided in Table \ref{tab:exp-parameters}.

\begin{table}[h]
\centering
\renewcommand{\arraystretch}{1.1}
\caption{Hyperparameters for QPT}
\label{tab:exp-parameters}
\begin{tabular}{cccc}
\toprule
Parameter   & All tasks  & Parameter   & All tasks  \\ \midrule
Number of layers          & 3 & Number of attention heads & 8  \\
Embedding dimension       & 128 & Batch size  & 2048 \\
Context length $K$        & 10 & Learning rate  & 0.0001  \\
Droupout                  & 0.1 & Adam betas  & (0.9, 0.999) \\
Grad norm clip            & 0.25 & Cost threshold  & 10 \\ 
Training steps (BulletGym, SafetyGym) & 100000 & Training steps (MetaDrive) & 200000 \\
\bottomrule
\end{tabular}
\end{table}

\subsection{Ablation of the number of candidate actions in Ensemble}
\label{sec:abnumber}

In our implementation, we utilize a default ensemble size of 50 candidate actions during inference. 
We also conduct a systematic ablation study specifically examining the impact of ensemble size on performance using the MetaDrive harddense environment as our testbed. 
The results in Table \ref{tab:exp-number} reveal a nuanced relationship between ensemble size and overall performance. As the number of sampled candidates increases from small values, we observe consistent performance improvements, indicating that larger ensemble sizes enable more comprehensive exploration of the action space and higher-quality policy selection. However, this relationship exhibits clear non-monotonicity, with performance plateauing and eventually declining beyond a certain threshold. 
A larger number of candidate target pairs provides a broader search space, potentially improving performance. However, this also incurs increased computational costs and greater susceptibility to noisy or suboptimal pairs, stemming from the biased estimation of the learned Q-networks.

\begin{table}[h]
\centering
\renewcommand{\arraystretch}{1.1}
\caption{Impact of the number of candidate target reward and cost pairs in the \textit{harddense} task in the \textit{MetaDrive} setting.}
\label{tab:exp-number}
\begin{tabular}{cccccc}
\toprule
   & 1  & 10   & 30 & 50 & 100  \\ \midrule
Reward & $0.49 \pm 0.04$ & $0.50 \pm 0.02$ & $0.52 \pm 0.03$ & $0.50 \pm 0.02$ & $0.48 \pm 0.04$ \\
Cost & $0.84 \pm 0.02$ & $0.83 \pm 0.03$ & $0.90 \pm 0.05$ & $0.81 \pm 0.03$ & $0.80 \pm 0.03$ \\
\bottomrule
\end{tabular}
\end{table}

\end{document}